\documentclass[accepted]{uai2022}

\usepackage[american]{babel}

\usepackage{natbib} 
\usepackage{mathtools} 
\usepackage{booktabs} 
\usepackage{tikz} 

\title{Instructions for Authors: Title in Title Case}

\usepackage{notations}
\usepackage{microtype}
\usepackage{graphicx}

\usepackage{hyperref}
\usepackage{url}

\usepackage{amsfonts}       
\usepackage{amsthm}
\usepackage{nicefrac}       
\usepackage{microtype}      
\usepackage{xcolor}         
\usepackage{subcaption}
\usepackage{mathtools}
\usepackage{comment}

\newtheorem{defn}{Definition}
\newtheorem{prop}{Proposition}

\newcommand{\michel}[1]{{\color{red} \textbf{Michel}:#1}}
\renewcommand{\michel}[1]{}

\newcommand{\bx}{\boldsymbol{x}}
\newcommand{\by}{\boldsymbol{y}}
\newcommand{\bs}{\boldsymbol{s}}
\newcommand{\be}{\boldsymbol{e}}
\newcommand{\bp}{\boldsymbol{p}}

\newcommand{\btheta}{\boldsymbol{\theta}}
\newcommand{\balpha}{\boldsymbol{\alpha}}
\newcommand{\parents}{\textbf{Pa}}
\newcommand{\rf}{{\rm ref}}

\newcommand{\bg}{{\rm \textbf{g}}}

\newcommand{\Bf}{{\rm {\bf f}}}

\newcommand{\G}{G}
\newcommand{\X}{\mathcal{X}}
\newcommand{\R}{\mathbb{R}}

\newcommand{\bernhard}[1]{\textbf{\color{red}~B:}{~\color{blue}#1}}
\renewcommand{\bernhard}[1]{}

\title{Learning soft interventions in complex equilibrium systems}

\author[1]{\href{mailto:<mbesserve@gmail.com>?Subject=Your UAI 2022 paper}{ Michel Besserve}}
\author[1]{Bernhard Sch\"olkopf}

\affil[1]{Department of Empirical Inference, Max Planck Institute for Intelligent Systems, T\"ubingen, Germany.} 
\begin{document}
\maketitle

\begin{abstract}
  Complex systems often contain feedback loops that can be described as cyclic causal models. Intervening in such systems may lead to counterintuitive effects, which cannot be inferred directly from the graph structure. After establishing a framework for differentiable soft interventions based on Lie groups, we take advantage of modern automatic differentiation techniques and their application to implicit functions in order to optimize interventions in cyclic causal models. We illustrate the use of this framework by investigating scenarios of transition to sustainable economies. 
\end{abstract}

\section{Introduction}
Designing optimal interventions in complex systems, composed of many interacting parts, is a key objective in multiple fields. In the context of socio-economic systems, the design of public policies to improve economic and social welfare is a major source of scientific and political debate. 
Moreover, the positive aspects of socio-economic activities need to be traded-off with their environmental impacts, as their long term consequences may considerably affect societies  \citep{dearing2014safe,sherwood2010adaptability}. Interestingly, a priori intuitive interventions in such systems may lead to paradoxical outcomes. The rebound effect in energy economy, first reported by \citet{jevons1866coal}, is paradigmatic: while the energy efficiency of devices may considerably increase due to technological  improvements, this may trigger an overall increase of energy consumption due to increases in demand \citep{brockway2021energy}. This suggests in particular that efficiency alone may not be the best way to foster a transition towards sustainability, and calls for a quantitative study of optimal interventions in such complex systems \citep{arrobbio2018}. As argued for the case of rebound effects \citep{wallenborn2018}, such unexpected behaviors may reflect balanced causal relationships designed by evolution \citep{andersen2013expect}, and feedback loops \citep{blom2021causality} that maintain a system at an optimal ``equilibrium'' operating point independent from external perturbations, challenging classical causal inference assumptions of faithfulness and acyclicity. 

While interventions have been extensively investigated theoretically in the field of causality \citep{pearl2000causality,imbens2015causal}, the case of systems incorporating feedback loops remains particularly challenging, and therefore led to only limited applications to real-life complex systems.
A possible first step to study such systems is to approximate them by a model that operates at an equilibrium point, and can thus be described by a cyclic structural causal model \citep{bongers2016foundations}. Such models satisfy a self consistent set of equations that, under unique solvability assumptions, fully identifies the operating point, and allows to study interventions. For practical and ethical reasons, interventions that do not change the causal structure, called soft interventions, arguably provide a more realistic account of changes that can be performed in real life systems.  While a restricted set of qualitative results exist for such interventions \citep{blom2020conditional}\michel{add more, bareinboim...}, their quantitative assessment and design in complex systems is made difficult by the analytical intractability of the self-consistency relations. 

In this paper, we propose a framework for a general class of differentiable parametric soft interventions based on Lie groups and leverage recent technical and algorithmic developments allowing learning implicit functional relationships \citep{bai2019deep} to optimize such interventions. After defining Lie interventions and assaying their theoretical properties, we provide a computational framework to optimize them. We illustrate its application to economic models derived from real data, offering a novel approach to computational sustainability. Proofs are provided in Appendix~\ref{app:proofs}. Code is available at \href{https://github.com/mbesserve/lie-inter}{https://github.com/mbesserve/lie-inter}.

\paragraph{Related work.} Various types of economic equilibrium models (EEM) have been used to investigate macroeconomic effect of specific interventions \citep{wiebe2018implementing,wood2018}. Also, experimental design in two-sided marketplaces has been investigated in \citep{johari2022experimental}. 
In contrast to such work, we develop a general optimization framework that allows the optimal design of interventions to achieve specific goals. A restricted set of EEMs have been investigated more extensively from an optimization perspective (see, e.g., \citealt{esteban2004computing}); however, these are restricted to rather specific assumptions and constraints that allow to address optimization with linear programming approaches. 
Instead, we rely on automated differentiation and backpropagation algorithms that allow studying mechanisms and interventions with a broad range of non-linearities. In the field of causality, several studies investigate the relationship between the equilibrium of dynamical systems and structural causal models (SCM) \citep{MooJanSch13,peters2020causal} and how the causal structure can be learnt from data. 
In contrast, we focus on designing soft interventions in an known SCM at equilibrium. 
While the specificity of soft interventions have started to be investigated theoretically in 
structural causal models \citep{rothenhausler2015backshift,kocaoglu2019characterization,jaber2020causal,correa2020general,blom2020conditional}, 
the present work is to the best of our knowledge the first to investigate theoretically and algorithmically the design of soft interventions in cyclic causal models. The algorithmic approach relies on modeling economic equilibrium with deep equilibrium models \citep{bai2019deep}. This approach belongs to the category of implicit deep learning \citep{el2021implicit}, which has been used in a variety of applications such as model predictive control \citep{amos2018differentiable} and multi-agent trajectory modelling \citep{geiger2020learning}. \michel{reduce?}

\section{Motivation and Background}
\label{sec:backgrnd}
\subsection{Environmental Economic models}\label{sec:mrio}
In the face of the increasing severity of climate change and further environmental impacts of human activities, our societies face challenges to transition to more sustainable economies. An overarching difficulty 
is the complexity of the systems that need to be intervened on, which comprise tightly intertwined components, ranging from economic agents to a broad variety of ecosystems \citep{haberl2019contributions}. 

A classical way to represent the economy and its impacts are input-output (IO) multi-sector economic equilibrium models \citep{stadler2018exiobase}, in which economic activities are divided in $d$ interdependent \textit{sectors} and described by a positive $d$-dimensional \textit{output} vector $\bx$ (see Appendix~\ref{app:back}).
We take as a guiding example the demand-driven model introduced by \citet{leontief1951structure}, which is the basis of the \textit{Input-Output analysis} approach to environmental impact assessment. In such models, the sectors' outputs at economic equilibrium $\bx^*$ are dependent on the vector $\by$ gathering final demand for each product (consumed by users instead of being used to make another product). Satisfying the demand of all sectors implies the self-consistent equation
\begin{equation}\label{eq:leontief}
	\bx^*=A\bx^* +\by\,,
\end{equation}
where $A$ is the so-called  \textit{technical  coefficients  matrix}, with $A_{ij}$  the amount  of  each  product $i$ used as input to  produce product $j$. 
An example of technical coefficient matrix estimated from economic data is provided in Fig.~\ref{fig:leontief}.
While such equilibria can be thought of as the asymptotic value of $\bx$ in a dynamic model (see~Appendix~\ref{app:back})
we focus our analysis on the equilibrium equations without consideration for the dynamics that gives rise to it. 
In turn, the socio-economic impacts (e.g., employment) and environmental stressors (e.g., GHG emissions, water use, ...) of each sector's activity is gathered in a vector of \textit{impacts} $\bs$ such that
\begin{equation}\label{eq:impact}
	\bs = R\bx
\end{equation}
where $R$ is a \textit{footprint intensity} matrix such that $R_{ij}$ is the amount of impact of type $i$ generated by unit of output $j$. 
To mitigate major long term negative consequences of environmental stressors (see, e.g., \citealt{dearing2014safe,sherwood2010adaptability}), a reorganization of the global economy is required, which may consist in intervening on economic sectors, their impacts and their interactions reflected in the matrices $A$ and $R$. However, this faces three challenges. 

\begin{figure*}
	\begin{subfigure}{.3\linewidth}
		\includegraphics[width=\linewidth]{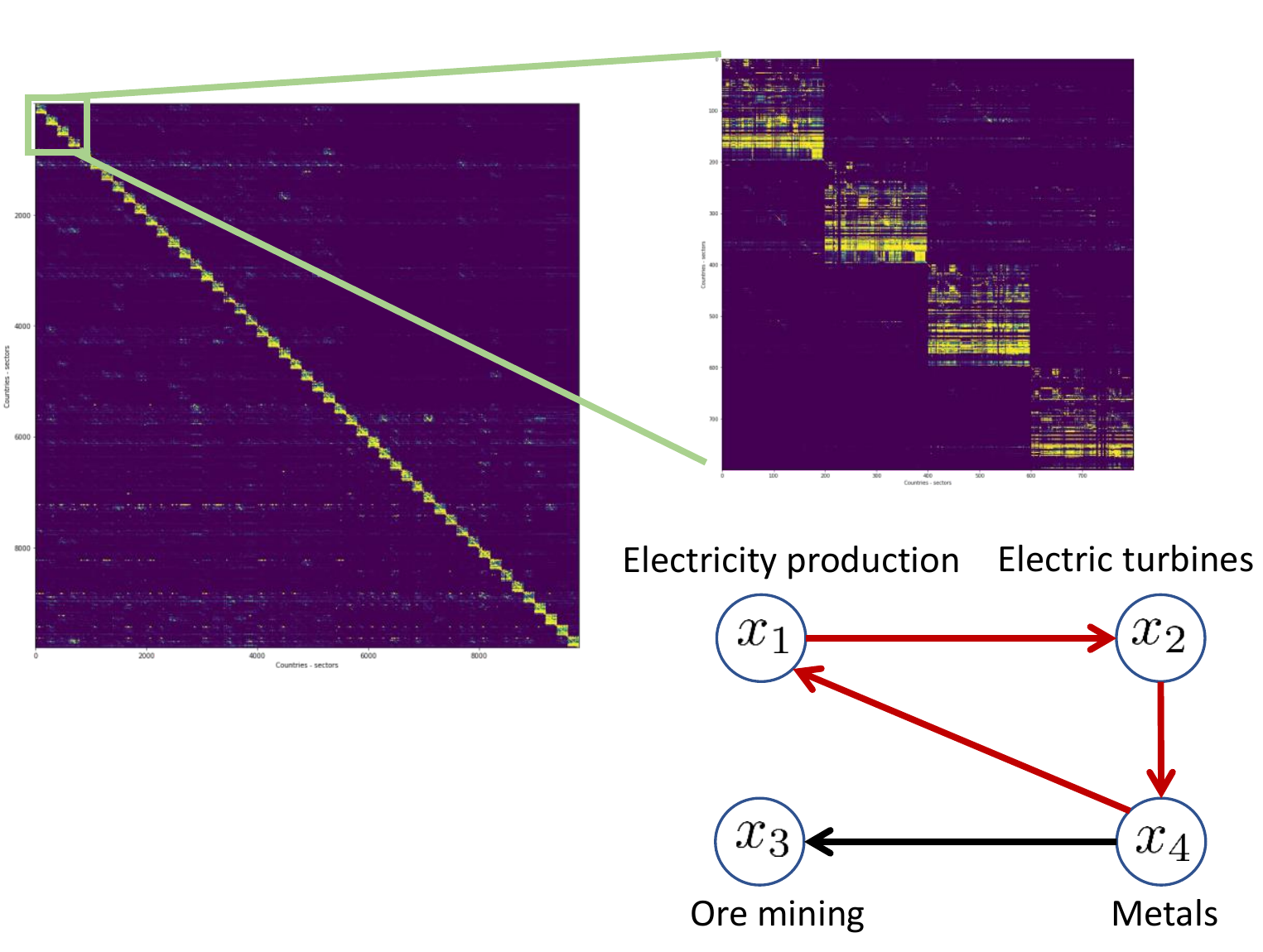}
		\subcaption{\label{fig:leontief}}
	\end{subfigure}
	\hfill
	\begin{subfigure}{.31\linewidth}
		\includegraphics[width=\linewidth]{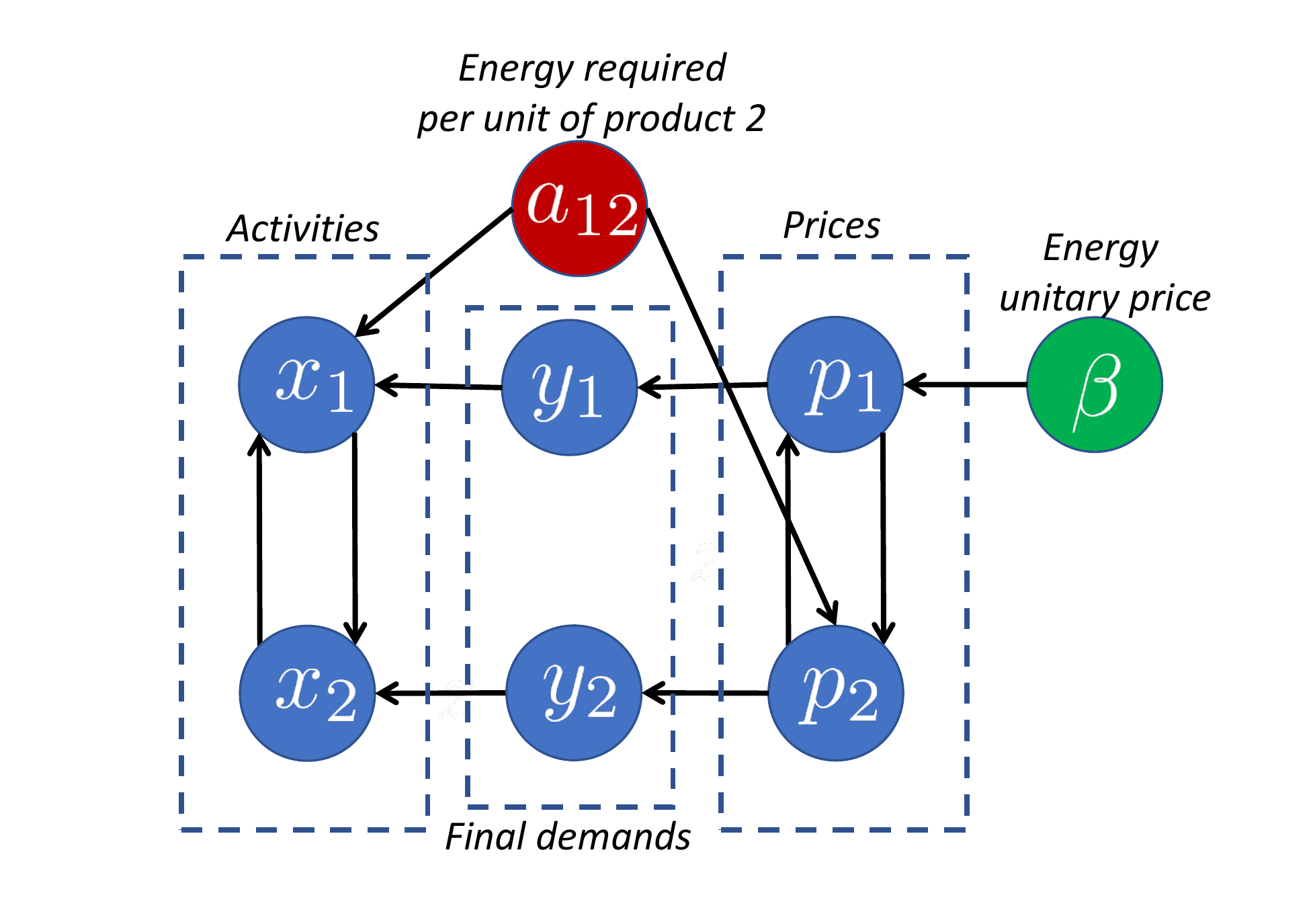}
		\subcaption{\label{fig:pricemod}}
	\end{subfigure}
	\hfill 
	\begin{subfigure}{.37\linewidth}
		\includegraphics[width=\linewidth]{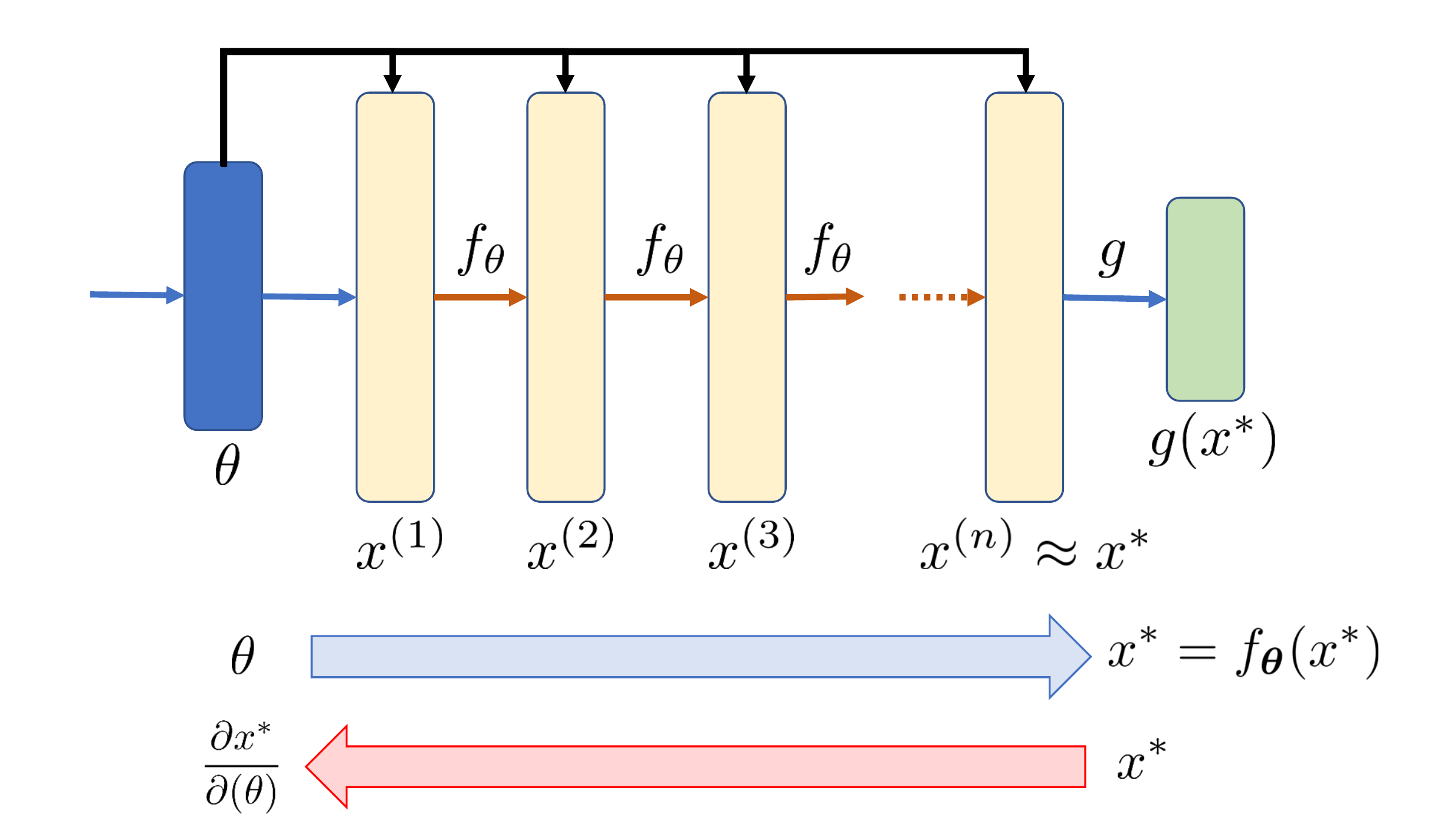}
		\subcaption{\label{fig:deepEq}}
	\end{subfigure}
	\caption{(a) Top left: technical coefficient matrix between 200 sectors and 49 world regions for 2011 (source: Exiobase 3, \citealt{stadler2018exiobase}). Top right: magnification of the top left corner of this matrix. Diagonal blocks reflect the stronger sector dependency within a country. Bottom right: putative example of cyclic dependency between different sectors. (b) Illustration of the causal graph for \textit{rebound trough prices} in a two sector economy. (c) Principle of deep equilibrium models.
	\label{fig:1}}
\end{figure*}

\paragraph{\textit{Challenge 1}: social acceptability.} Reducing a sector's activity may lead to both positive environmental effects (yielding lower footprints) and negative socio-economic impacts (such as reducing economic growth and employment, see Appendix~\ref{sec:supdisc}). Decision makers thus trade off environmental goals with the social acceptability of the chosen policies. 

\paragraph{\textit{Challenge 2}: recurrence between sectors.} The sectors' activities are tightly intertwined by their reciprocal demands, as illustrated 
at the bottom of Fig.~\ref{fig:leontief}: electricity production through renewable energy requires wind turbines, which require metals, while the metal industry requires itself electricity to extract metals from ores and transform them. Such cycles make it challenging to anticipate the system-wide consequences of interventions a particular sector. 

\paragraph{\textit{Challenge 3}: rebound effects.} The complexity of the economic system also manifests itself through balancing mechanisms that reflect the utility maximization behavior of economic agents, such as rebound effects. Consider $\bx^*$ in eq.~\ref{eq:leontief}, which can be written as a function of final demand
\begin{equation}\label{eq:leoninv}
\bx^* = (I-A)^{-1} \by\,.
\end{equation}
In practice, final demand is influenced by prices of each good and often modeled by a static demand curve $d_i$ for good $i$ such that
$y_i = d_i(p_i)$. 
A final demand rebound through prices can be simulated in the Leontief model as follows. 
Energy efficiency of the production of a particular good $j$ corresponds to a decrease of $A_{ej}$, where $e$ indicates the energy sector, but this modification also affects the unit price through energy costs. For simplicity, we define the price vector $\bp$ of goods such that it is proportional to the energy required in all sectors involved in the production of one unit of this good. It can thus be modelled by a self-consistent relation involving the technical coefficient matrix:
\[
\bp^* = A^\top \bp^* +\beta \boldsymbol{\delta}_e\,,
\]
where $\delta_e$ is a canonical basis vector which takes value $1$ for the energy sector, and value $0$ for all other sectors. For illustrative purposes, the overall causal model 
is shown in Fig.~\ref{fig:pricemod} in the case of a two sector economy, with sector 1 being the energy sector. The price-based rebound mechanisms then operates as follows: a decrease of $A_{ej}$ will decrease energy demand on sector $e$, but will also decrease the unit price of goods for sector $j$ (and downstream sectors consuming its goods). 
Because the demand curves $d_j$ are monotonically decreasing, the price drop increases the final demand for these products, which in turn increases economic activity according to eq.~\ref{eq:leoninv}, and their environmental footprint. 
The rebound may thus be avoided by simultaneously intervening on the unit price of energy $\beta$ through a tax policy, so that price level is maintained high and prevents increases of final demand (see Fig.~\ref{fig:pricemod}). Importantly, while eq.~(\ref{eq:leoninv}) provides a linear relationship between activity and final demand, once we assume $\bp$ is price dependent, the system of equations becomes non-linear and finding an analytic expression of the economic equilibrium is nontrivial.
Our approach to designing interventions in cyclic causal models will be applied to models illustrating the above three challenges. 

\subsection{Cyclic causal models}
Interventions and their effects on systems have been investigated using Structural Causal Models (SCM) \citep{pearl2000causality}. In this framework, relationships between observed variables $X_k$ are described by a set of structural assignments 
\[
X_k \coloneqq f_k(\parents_k,\epsilon_k)\,,
\]
where $\parents_k$ indicates the parents of variable $X_k$ in an associated directed causal graph, such as the one illustrated in Fig.~\ref{fig:leontief}. Interventions turn an SCM into a different one, by applying a modification to at least one of its elements. Broadly construed, interventions range from ``hard'' interventions that modify the structure of the graph to ``soft'' interventions that do not \citep{eberhardt2007interventions}. 

While in acyclic graphs, interventions have generic effects on their descendants in the causal graph, and no effects on the parents, \cite{blom2020conditional} have shown that causal effects are less easy to read in graphs containing cycles. 
Anticipating the effect of interventions in cyclic graphs overall requires to estimate the changes in the equilibrium point, which is typically non-trivial. While a variety of approaches may be used (e.g., based on root finding approaches), designing optimal interventions for self-consistent equations that cannot be handled analytically is challenging, especially in high dimensional systems. Recent work in deep neural network has come up with techniques allowing gradient descent based optimization of such equilibrium models \citep{bai2019deep}. 


\subsection{Deep equilibrium models}

 \Citet{bai2019deep} introduced deep learning architecture elements with input-output functional relationships $\bx^*=\bg(\btheta)$ between variables $\bx^*$ and parameters $\btheta$ that are only defined through 
 a self-consistent equation 
\[
\bx^*=\Bf_{\btheta} (\bx^*).
\]
Assuming that for each value $\btheta$ %
there is a unique solution $\bx^*$, the gradient with respect to one parameter component $\theta_k$ can be obtained through another self-consistent equation
\begin{equation*}
\frac{\partial \bx^*}{\partial \theta_k} = \frac{\partial \Bf_{\btheta}}{\partial \theta_k}(\bx^*)+\frac{\partial \Bf_{\btheta}}{\partial \bx}\frac{\partial \bx^*}{\partial \theta_k} \,.
\end{equation*}
Overall, $\bg$ can be integrated as a layer in more complex differentiable models, which, as depicted in Fig.~\ref{fig:deepEq}, can be understood as a cascade of multiple layers with identical functions and shared parameters, with specific accelerated fixed point iteration approaches to compute the forward and backward passes \citep{bai2019deep}. In this paper, we use Anderson's acceleration \citep{walker2011anderson}, which essentially generalizes the forward iteration approach (i.e. iterating $\bx_{k+1}=\Bf(\bx_k)$ until convergence) by leveraging the $m$ previous estimates in order to find a better estimate. Overall, these layers offer a differentiable framework for investigating the behavior of cyclic graphs that we use to design interventions.


\subsection{Lie groups}
Using deep equilibrium models, we can learn differentiable soft interventions compatible with classical optimization frameworks. We will use the concept of Lie groups, which are smooth manifolds of transformations (see Appendix~\ref{app:back} for more background), 
in order to implement smooth soft interventions. 
In short, a group $\G$ is a set of objects equipped with a group ``multiplication'' operation mapping $(g_1,g_2)\in\G^2$ to $g_1 g_2\in\G$ and an inverse operation $g^{-1}$ with the following properties: 
\begin{itemize}
\itemsep0em
\parskip0em
	\item (associativity) $(g_1 g_2) g_3 = g_1 (g_2 g_3)$,
	\item (identity element) there exist a unique identity element $e$ such that for all $g$, $eg=ge=g$,
	\item (inverse) for all $g\in \G$, there exists a unique element $g^{-1}$ such that $g g^{-1}=g^{-1} g=e$.
\end{itemize}
\vspace{-1\topsep}
Groups perform transformations on objects in a set $\X$ through the definition of a group action operation $\varphi$ mapping $(g,x)\in \G\times \X$ to $\varphi_g(x)= g\cdot x\in \X$, such that 
\vspace{-\topsep}
\begin{itemize}
\itemsep0em
\parskip0em
	\item (identity) for all $x\in \X$, $e\cdot x=x$,
	\item (compatibility) for all $(g,h)\in \G\times\G$, for all $x\in \X$, $g\cdot (h\cdot x)=(gh)\cdot x$.
\end{itemize}
\vspace{-\topsep}
A real Lie group is a group that is also a finite-dimensional real smooth manifold (see Appendix~\ref{app:back}), in which the group operations of multiplication and inversion are smooth maps. 
The differentiability of Lie groups will be leveraged to design smooth interventions.

\section{Intervening in smooth models}
\subsection{Smooth causal graphical models}

We define a smooth structural causal model (SSCM)
as a set of variables $\{x_j\}$ related to each other through structural equations and vertices in a directed graph as follows.
\begin{defn}[SSCM]\label{def:SCM}
	A $d$-dimensional smooth structural causal model is a 4-tuple $(\mathcal{X},\mathcal{T},\mathbb{S},\mathcal{G})$ consisting of
	\vspace{-\topsep}
	\begin{itemize}
	\itemsep0em
	\parskip0em
		\item two collections of smooth manifolds $\mathcal{X}=\{\mathcal{X}_i\}_{i=1..d}$ and $\mathcal{T}=\{\mathcal{T}_j\}_{j=1..d}$ ,	
		\item a directed graph $\mathcal{G}=(V,E)$ with set $V$ of $d$ vertices and set $E$ of directed edges between them, each vertex being associated to one variable $x_i\in\mathcal{X}_j$ ,
		\item a set $\mathbb{S}$ of structural assignments 
		$
		\{x_j \coloneqq f_j(\parents_j,\theta_j),  \theta_j \in\mathcal{T}_j\}_{j=1,\dots,d}\, ,
		$
		where $f_k$ are smooth maps, and $\parents_j$ are the variables indexed by the set of parents of vertex $j$ in $\mathcal{G}$.
	\end{itemize} 
	\vspace{-1\topsep}
\end{defn}
Compared to classical definitions of SCMs (see, e.g., \citealt{causality_book}), we have replaced exogenous random variables by deterministic parameters living on a manifold. This general definition does not prevent assigning random variables to some (components of) these parameters. In the cases considered here, $\mathcal{T}_i$ are subsets of Euclidean spaces. We are particularly interested in cyclic SCMs, where there exists at least one directed path linking one vertex to itself. As a consequence, the possible values achieved by each variable have to be chosen among the solutions of the $d$ self-consistent structural equation constraints. We assume the unintervened causal model is locally uniquely solvable.
\begin{defn}
	A SSCM is locally uniquely solvable around a reference point $(\bx^{\rf},\btheta^{\rf})$ whenever there exists a neighborhood $U_{\btheta}$ of $\btheta^{\rf}$ and a neighborhood $U_{\bx}$ of $\bx^{\rf}$ such that for all $\btheta\in U_{\btheta}$ there exists a  unique (self-consistent) solution to the set of structural assignments $\bx^*(\btheta)\in U_{\bx}$.  
\end{defn}
Note that this is adapted to our SSCM definition and differs from the unique solvability definition of \cite{bongers2016foundations}, which was formulated for causal models with random exogenous variables.  
This property is guaranteed by a condition on the Jacobian of the structural equations.
\begin{prop}\label{prop:localSolv}
	We say the SSCM is locally diffeomorphic at $(\bx^{\rf},\btheta^{\rf})$ when $(\bx^{\rf},\btheta^{\rf})$ is a solution and the Jacobian of the mapping $\bx\rightarrow \bx-\Bf(\bx,\btheta^{\rf})$ is invertible. Such a SSCM is uniquely solvable around this reference point and the local mapping	$\btheta \mapsto\bx^*(\btheta)$
	is smooth. 
\end{prop}

In the context of IO analysis presented in Section~\ref{sec:mrio}, the variables can be the sector's outputs and unit prices. For eq.~(\ref{eq:leontief}), the resulting SSCM thus contains the affine structural assignments associated to each component of $\bx$ 
\[
\mathbb{S}=\{x_k\coloneqq\sum_j A_{kj} x_j +y_k\}\,,
\]
which are clearly smooth, and the $\{A_{kj},y_k\}$'s may be assumed fixed or free parameters within an interval.

\subsection{Lie interventions}
We will consider interventions parameterized by an element $u$ that turns the unintervened equilibrium solution $\bx^*(\btheta)$ into the \textit{intervened equilibrium solution} $\bx^{(u)}(\btheta)$ over a range of values of $\btheta$. In particular,
we define Lie interventions implemented through the action of a Lie group. 
\begin{defn}[Lie intervention]\label{def:lieInter}
	A Lie intervention on an SSCM $\mathcal{M}=(\mathcal{X},\mathcal{T},\mathbb{S},\mathcal{G})$ with a set of smooth structural assignments $\mathbb{S}$ is a pair $(L,\varphi)$ where $L$ is a Lie group and  a smooth group action $\varphi\,:\,L\times \mathbb{S}\rightarrow \mathbb{S}$. The action defines a family of intervened SSCMs $\mathcal{M}^{(g)}=(\mathcal{X},\mathcal{T},\varphi(g,\mathbb{S}),\mathcal{G})$, for $g$ in a neighborhood of the identity within $L$. 
\end{defn}
Note in particular that applying the identity element of the group leads to the original (unintervened) causal model.
Such interventions preserve unique solvability. 

\begin{prop}[Solvability]\label{prop:liesolv}
	For a Lie intervention on a locally diffeomorphic SSCM, there is a neighborhood $U_{L}$ of the identity $e$ in $L$ such that the intervention is soft, the family of intervened SSCMs is locally uniquely solvable and the local mapping to the intervened solution $(g,\btheta) \mapsto \bx^{(g)}(\btheta)$ 
	is smooth.
\end{prop}

\paragraph{Multiplicative Lie interventions.}
A simple way of intervening on an arbitrary system is to multiply one selected assignment by a strictly positive scalar coefficient. We can consider $\R^*_+$ equipped with multiplication as a Lie group, that acts on  a node by rescaling its structural assignment. 
Several such \textit{scalar} Lie interventions can then be combined into a \textit{distributed} intervention on a set of nodes instead of a single one. A group element is then a strictly positive vector $\balpha>0$ acting on assignments indexed by $I$ such that
\[
\balpha \cdot \mathbb{S}_{| I} = \{x_k \coloneqq \alpha_k f_k(x_k,\theta_k), k\in  I\}\,.
\]
In the context of Input-Output models presented in Section~\ref{sec:mrio}, applying this intervention can be seen as reducing or increasing the demand for products of specific sectors. Reducing the demand for a sector with large GHG emissions is for example a relevant objective for the transition to a sustainable economy and may be implemented by public policy in various ways (taxes, norms, ...). Such interventions are investigated in  industrial ecology \citep{wood2018}.

In the context of our guiding example, the influence of multiplicative interventions has an intuitive real world interpretation. However, \textit{shift interventions} (using the additive group, acting by addition on a structural assignment) may also be an easily interpretable choice in some settings, and have been exploited for causal inference \citep{rothenhausler2015backshift}. Moreover, some settings may require other classical, possibly multidimensional, Lie groups (e.g. \citet{besserve2018aistats} exploit the group of rotations of the $n$-dimensional Euclidean space $SO(n)$). 
Finally, in contexts where the model stems from a mechanistic model, e.g. relying on physics equations, Lie interventions that change meaningful model parameters may act on structural equations in more complex ways. 




\begin{figure*}
\begin{subfigure}{.32\linewidth}
	\includegraphics[width=\linewidth]{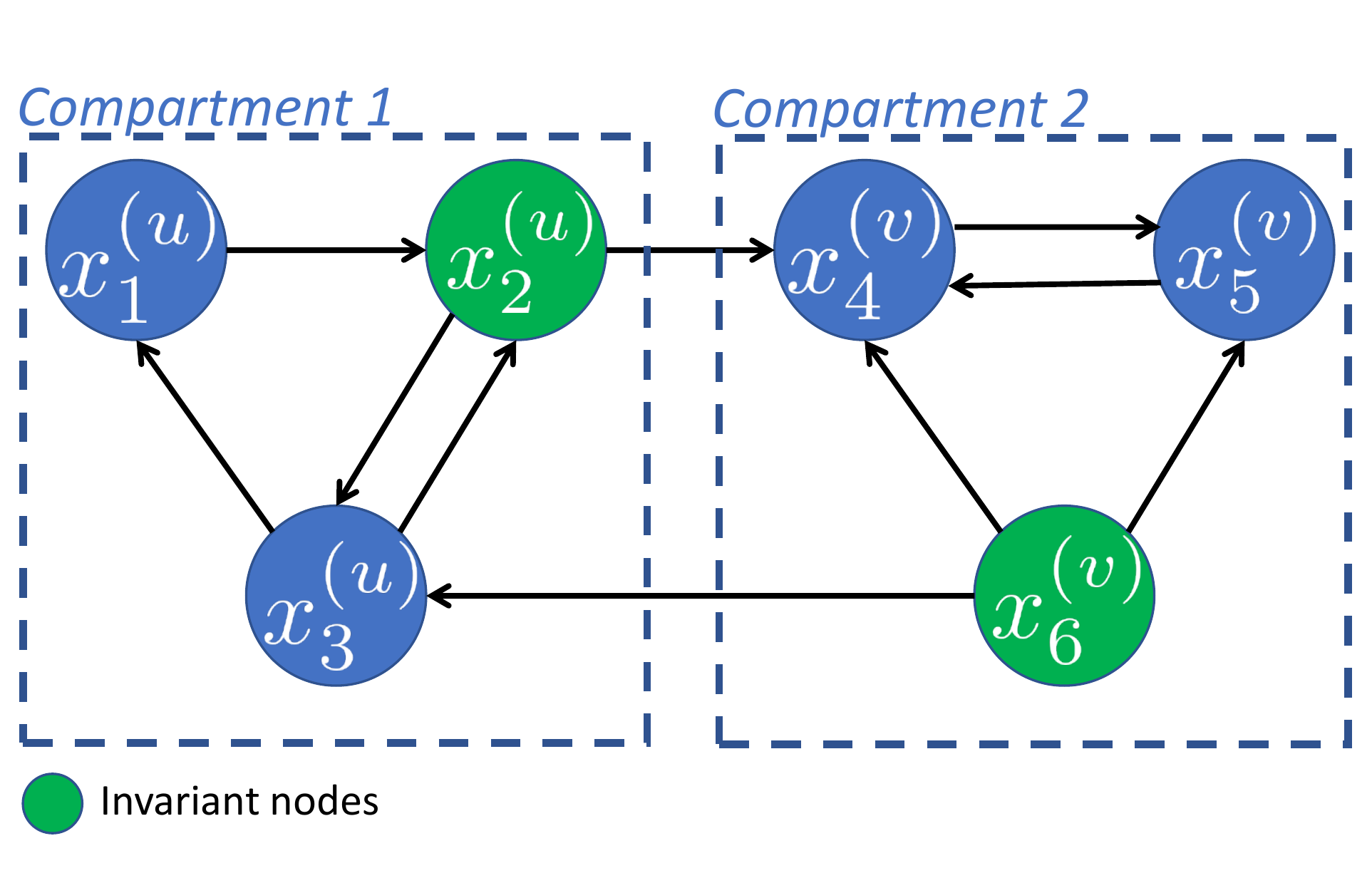}
	\subcaption{\label{fig:comparInter}}
\end{subfigure}
\hfill
		\begin{subfigure}{.33\linewidth}
	\includegraphics[width=\linewidth]{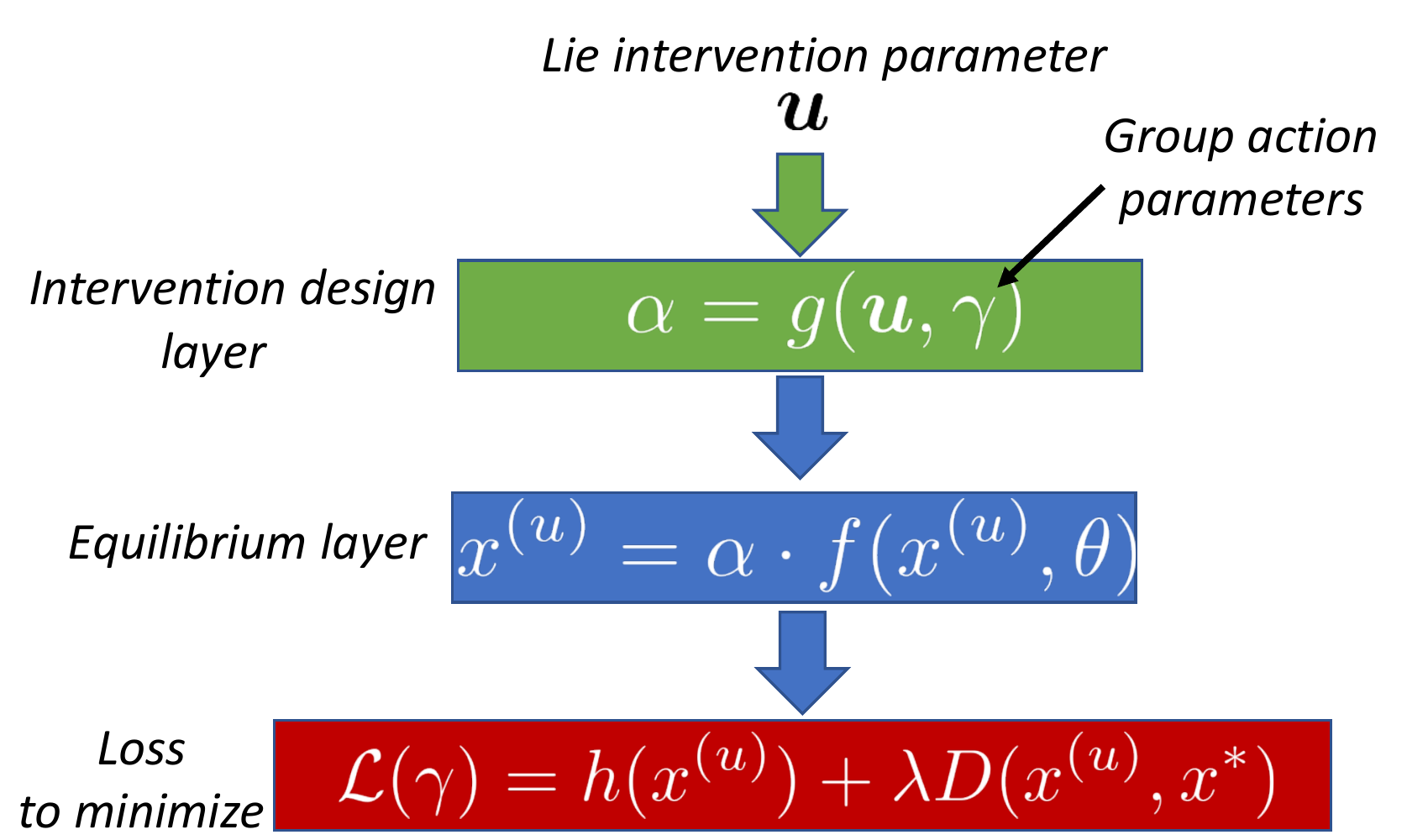}
	\subcaption{\label{fig:multOptim}}
\end{subfigure}
\hfill
	\begin{subfigure}{.32\linewidth}
	\includegraphics[width=\linewidth]{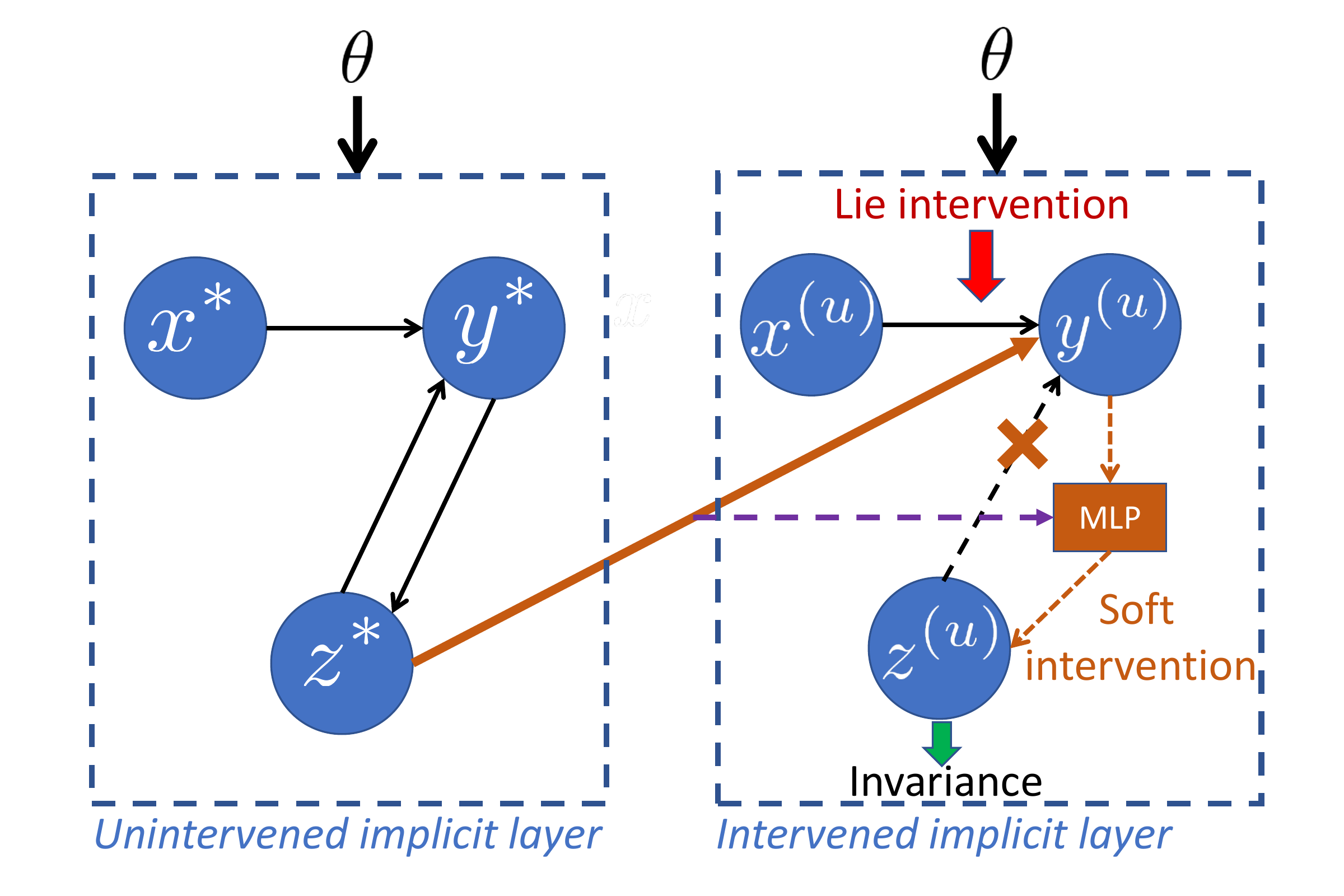}
		\subcaption{\label{fig:invarInter}}
	\end{subfigure}
	\caption{(a) Illustration of a compartmentalized intervention: enforcing invariance of the green nodes allows each compartment to be independently influenced by two (invariant) interventions $u$ and $v$. (b) Architecture for Lie intervention optimization. The equilibrium layer is controlled by intervention parameters and a loss is applied to its output. (c) Schematic representation of the procedure to learn invariant soft intervention ($y$: intervened node, $z$: invariant and auxiliary node). A multilayer perceptron (MLP) learns the soft intervention enforcing invariance of $z^{(u)}$ over a range of parameter values. \label{fig:design}}
\end{figure*}

\subsection{Invariant soft interventions}\label{sec:invar}
The rebound effect is paradigmatic of interventions that may trigger undesired effects that we wish to prevent. To this end, simultaneous interventions on other parts of a system have been considered in applications. For example, a rebound through prices can be prevented by a simultaneous auxiliary intervention of prices through taxes, such that the prices remain invariant to the overall intervention. Using the SSCM framework, we theoretically investigate the conditions under which some variables of the causal model can be maintained invariant to the Lie intervention on others.

\paragraph{Motivating example.}
Consider the following SSCM with parameters $\btheta=(\tau,\alpha,\beta,\gamma)$ with distributed multiplicative Lie intervention $\boldsymbol{u}$:
\vspace{-\topsep}
\begin{equation}\label{eq:motivex}
\setlength{\jot}{0pt}
\left\{
\begin{array}{ccl}
     	x &= &\tau \,,\\
	y &= &u_y (\alpha x +\beta z) \,, \\
	z &=& u_z \gamma y \,.
\end{array}
\right.
\end{equation}
By choosing $u_z=\frac{1}{u_y}$, the intervened equilibrium solution component $z^{(u)}$ becomes insensitive to multiplicative interventions $(u_y,u_z)$, such that $z^{(\boldsymbol{u})}(\btheta)=z^*(\btheta)$ for arbitrary values of parameters $\btheta$ in a neighborhood of the reference parameter (see Appendix~\ref{app:add}). This result suggests that the influence of soft interventions ($u_y$ in this example) can be restricted to a subset of nodes, by choosing a second intervention ($u_z$ in this example) on an auxiliary variable. However, it is unclear whether this result still holds when the functional assignment of $z$ becomes non-linear. 

To frame this question in a general setting, we introduce a class of soft interventions under invariance constraint.
\begin{defn}[Invariant soft interventions]
	Given an SSCM with Lie intervention from group $L$ on node $i$. The intervention leaves node $j$ \textit{invariant} by \textit{leveraging} node $k$ if for all group elements $u$ in a neighborhood $\mathcal{N}$ of the identity, there exists a soft intervention on node $k$, $f_k^{(u)}(\parents_k,\btheta)$, replacing functional assignment $f_k$ such that the intervened node value $x^{(u)}_j$ satisfies $x^{(u)}_j(\btheta) = x^*_j(\btheta)$ in a neighborhood of the reference parameter. Node $i$ is called the intervened node, node $j$ is called the invariant node, and node $k$ is called the auxiliary node.
\end{defn}
\paragraph{Remarks:} The soft intervention property is key, as it entails that the use of an auxiliary variable to enforce the invariance constraint must only exploit the information available to this node as defined by its parents in the unintervened graph (and no parameter values). This constraint makes deployment more realistic in a complex system, as intervening does not require supervision by an external entity monitoring the whole system. Unless otherwise stated, the auxiliary node will be chosen identical to the invariant node.

Let us denote $\bx_{-j}$ and $\Bf_{-j}$ the vector and mapping with the $j$-th component removed. We also define two quantities important for the existence of such interventions. The partial derivative $\frac{\partial x^*_j}{\partial x_k }_{|\btheta=\btheta^{\rf}}$ is obtained by performing a hard intervention $x_k=\lambda$ leading to equilibrium value $x^{(\lambda)}_j(\btheta^{\rf})$, and computing the derivative $\frac{d x^{(\lambda)}_j}{d \lambda}_{|\lambda = x^*_k(\btheta^{\rf})}$. The Jacobian $J^{\btheta}_{x^*_{\parents_k}}(\btheta^{\rf})$ is the Jacobian of the mapping from the parameters $\btheta$ to the vector consisting of the parent nodes of $k$ at equilibrium. Based on these two quantities, we have the following sufficient condition.
\begin{prop}\label{prop:invar}
	Consider an SSCM locally diffeomorphic at $(\bx^{\rf},\btheta^{\rf})$ with intervened/invariant/auxiliary triplet of nodes $(i,j\neq i,k\neq i)$. If the Jacobian of the mapping $\bx_{-j}\rightarrow \bx_{-j}-\Bf_{-j}(\bx_{-j},\btheta^{\rf})$ is invertible, $J^{\btheta}_{x^*_{\parents_k}}(\btheta^{\rf})$ has full column rank, and $\frac{\partial x^*_j}{\partial x_k }_{|\btheta=\btheta^{\rf}}\neq 0$, then the intervention on $i$ leaves node $j$ invariant by leveraging node $k$. 
\end{prop} 
This result suggests that the motivating example of eq.~(\ref{eq:motivex}) can be extended, in a neighborhood of the identity, beyond the linear case, when the number of free parameters considered remains low relative to the number of parents of the auxiliary node. However, as can be seen in the proof, the soft intervention on the auxiliary variable is given by an implicit function theorem, suggesting non-parametric models are necessary to learn it (based e.g. on automated differentiation methods).  This will be described in Sec.~\ref{sec:design}.

\begin{figure*}[h]
	\begin{subfigure}{.33\linewidth}
	\includegraphics[width=\linewidth]{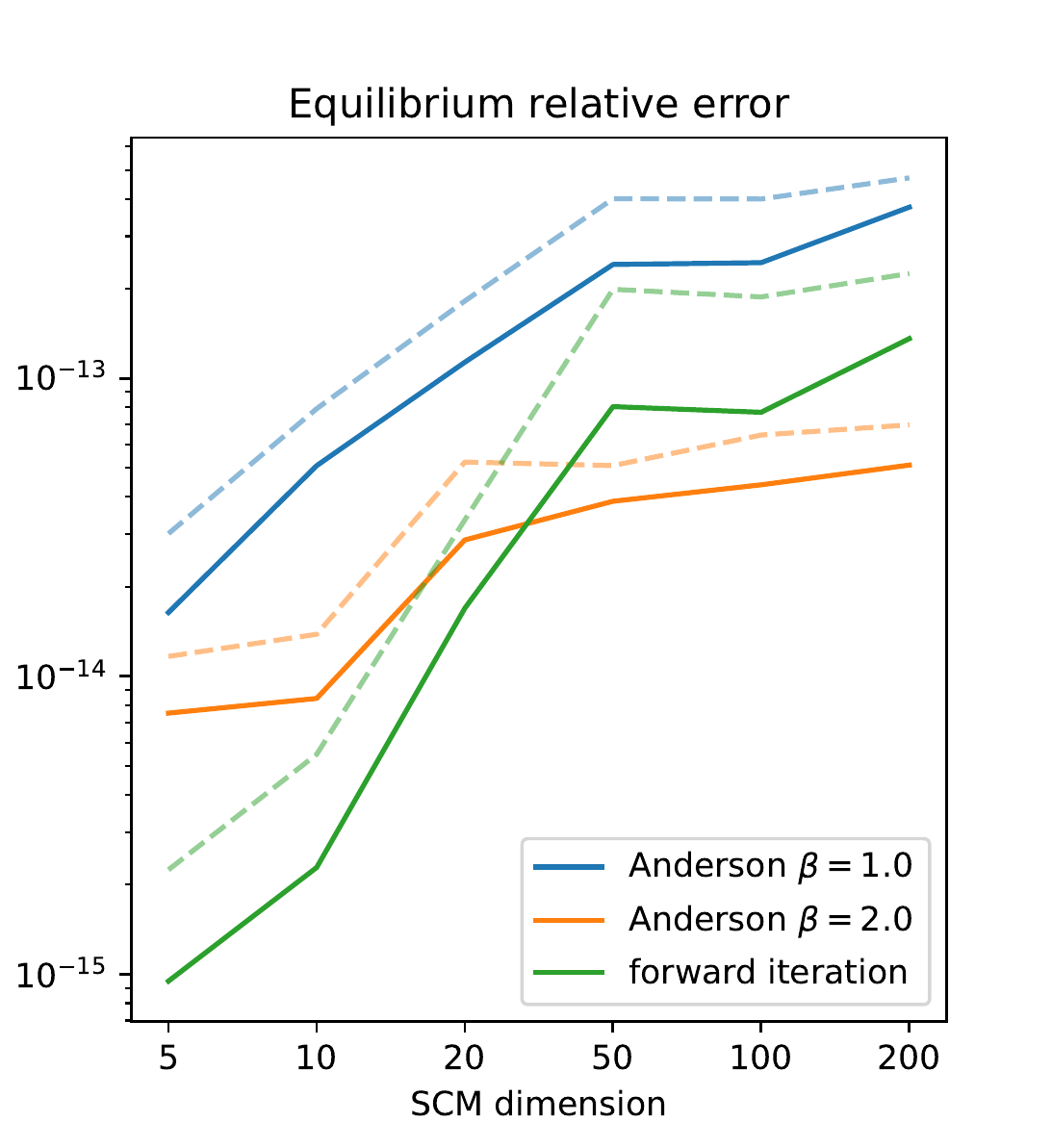}
		\subcaption{\label{fig:conv}}
	\end{subfigure}	
	\hfill
		\begin{subfigure}{.29\linewidth}
	\includegraphics[width=\linewidth]{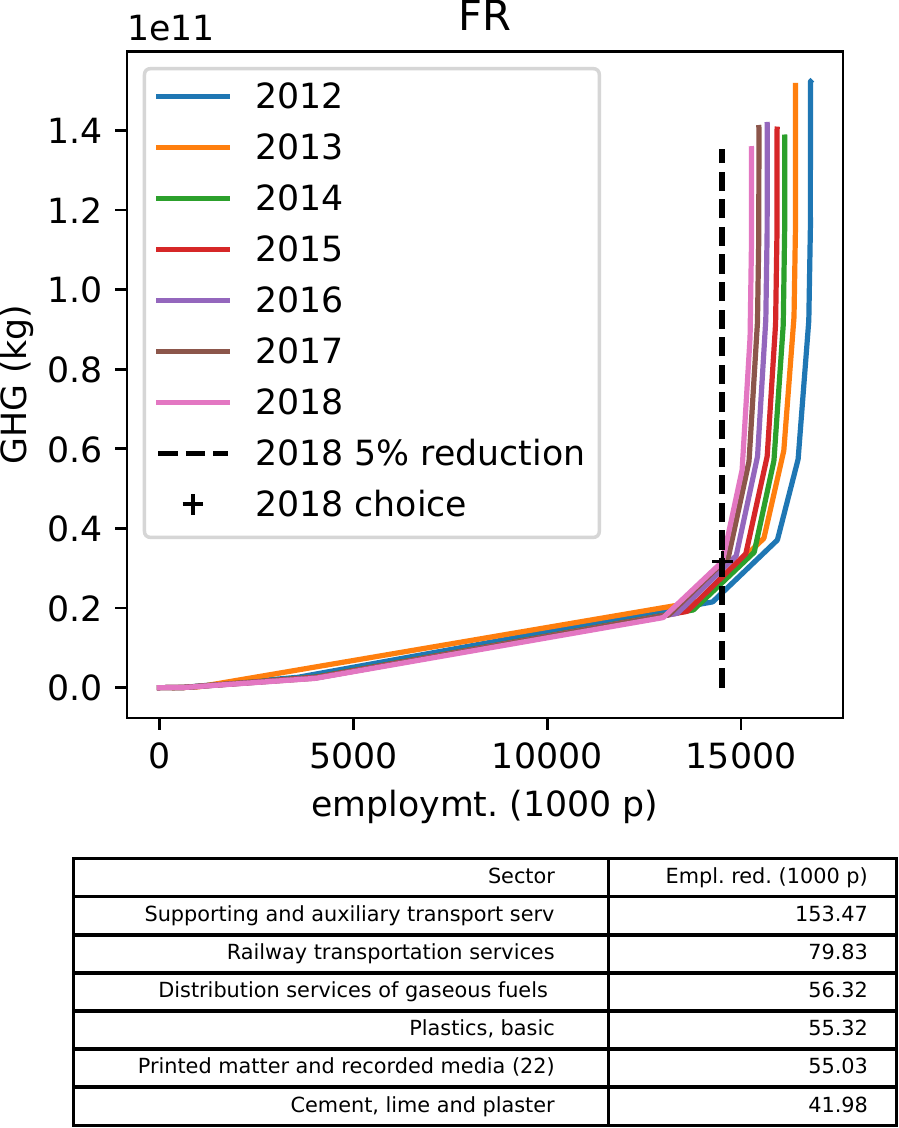}
		\subcaption{\label{fig:countOpt}}
	\end{subfigure}	
	\hfill
	\begin{subfigure}{.29\linewidth}
	\includegraphics[width=\linewidth]{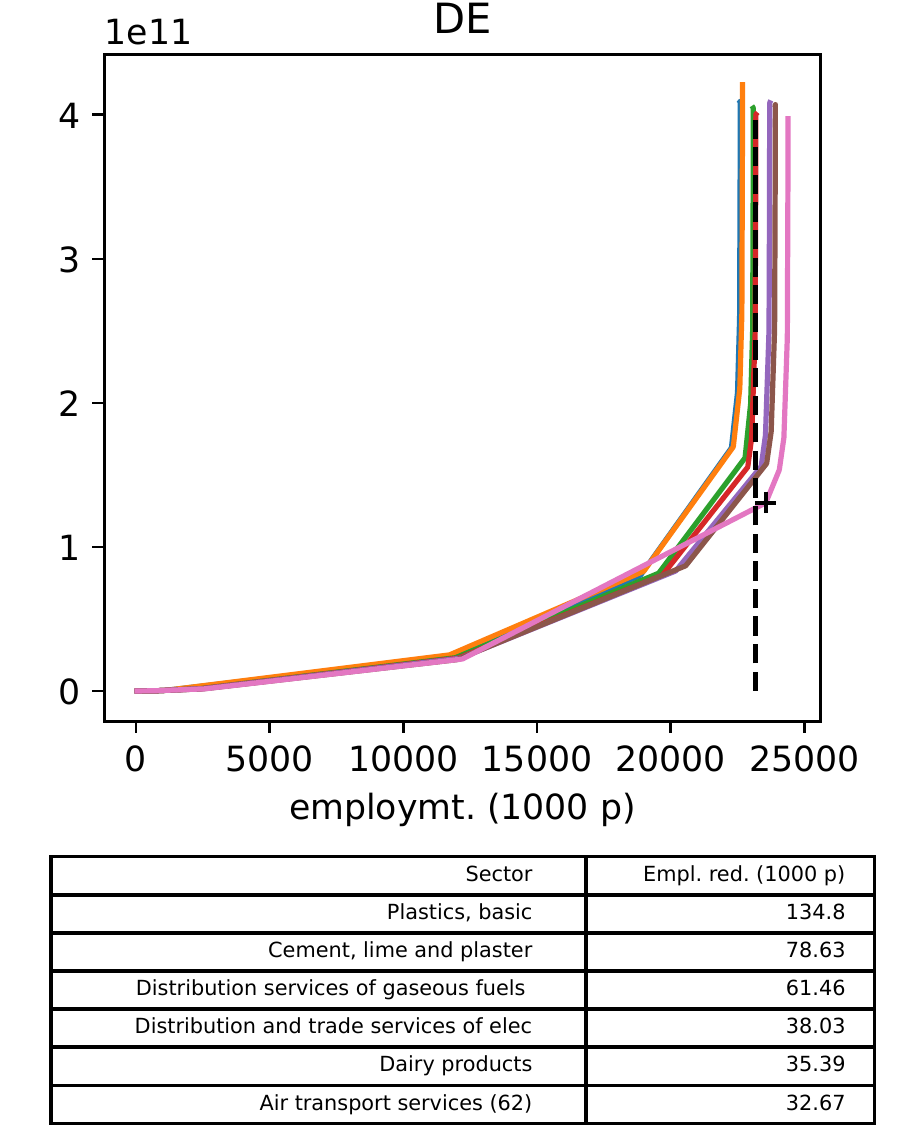}
		\subcaption{\label{fig:countOptDE}}
	\end{subfigure}	
	\caption{(a) Equilibrium relative error for different methods and SCM dimensions (solid: mean, dashed: mean+std). (b-c) Outcome of Lie intervention optimization on country models of GHG emission reduction in France (b) and Germany (c) for varying values of $\lambda$ in eq.~(\ref{eq:joboptim}), based on economic models estimated from different years. For year 2018, dashed lines indicates 5\% reduction in employment and the cross the corresponding $\lambda$ choice. Tables show sectors with largest employment reduction for 2018. \label{fig:exps}}
\end{figure*}

\subsection{Compartmentalized interventions}
Invariant interventions allow to restrict the propagation of effects to a subset of nodes. If a complex system can be partitioned into sparsely connected subsets of nodes, we can consider designing such interventions in order to modify the equilibrium values of each compartment  independently from each other. 
\begin{defn}
Given a partition of the SSCM nodes into $K$ compartments $\{C_k\}_{k=1,\dots,K}$. Given interventions on each compartment, parameterized by respective parameters
$u_k$, leading to the intervened SSCM equilibrium solution $\bx^{(u_1,\dots,u_k,\dots,u_K)}(\btheta)$. Interventions are compartmentalized when for all $k$, for all nodes
$j\in C_k$, component $\bx_j^{(u_1,\dots,u_k,\dots,u_K)}(\btheta)$
does not depend on $u_m$
for $m\neq k$.
\end{defn}
The following result guarantees that if the nodes influencing other compartments are made invariant, interventions on each compartment can be designed and performed independently from each other as their effects remain confined to their own compartment.

\begin{prop}\label{prop:compart}
	Given a partition $\{C_k\}$ of the SSCM nodes. If for each compartment $k$ there exists one invariant soft intervention performed on structural equations such that intervened, auxiliary and invariant nodes belong to the compartment, and all nodes of this compartment that have an outgoing arrow pointing to a different compartment are invariant, then those interventions are compartmentalized. 
\end{prop}
\michel{following is unclear, maybe point to the full column rank condition limiting the number of free parameters}A fundamental aspect of this result is that, from the definition of invariant interventions, compartmentalization is valid over a range of parameters of the causal model (a neighborhood of the reference point) and a range of Lie interventions parameters (a neighborhood of the identity). This can be seen as a  
way to enforce interpretability of interventions by restricting their influence to a specific subsystem, at least for a range of parameter values.  An illustration of a setting compatible with Prop.~\ref{prop:compart} is provided in Fig.~\ref{fig:comparInter}, where the equilibria of two sparsely connected compartments are interdependent (notably, the causal ordering algorithm described in \citet{blom2020conditional} returns a single cluster merging both compartments). Enforcing invariance of the green nodes, each associated to one intervention ($u$ and $v$) within their compartment allows applicability of Prop.~\ref{prop:compart}.

\section{
Intervention design}\label{sec:design}
To address \textit{Challenge 2} of Sec.~\ref{sec:mrio}, we design interventions with implicit layers (see Appendix~\ref{app:meth} for additional details). 
\paragraph{Differentiable architecture.}
Base optimization relies on a differentiable architecture comprising one central module representing the cyclic SSCM. Essentially, the cyclic model is represented by an equilibrium layer following \cite{bai2019deep}, schematized in Fig.~\ref{fig:deepEq}: the differentiable module is designed such that forward and backward passes through the equilibrium layer use Anderson acceleration to solve a fixed point equation. This equilibrium layer is cascaded if necessary with parametric layers to achieve specific goals. The architectures are implemented using the PyTorch library.

\paragraph{Lie intervention optimization}
We design an architecture around the equilibrium module to optimize multiplicative intervention according to a loss, as represented in Fig.~\ref{fig:multOptim}. Parameters $\boldsymbol{u}$ of the Lie group element are optimized in order to minimize an objective $\mathcal{L}(\bx^{(u)})$ achieved by the equilibrium solution of the SSCM. 
This objective may include an additional regularization term, $D(\bx^{(u)},\bx^*)$ with regularization parameter $\lambda$, to enforce that some properties of the intervened system remain invariant or close to the original, non-intervened, equilibrium solution $\bx^*$. 

\paragraph{Learning invariant interventions.}
In order to enforce invariance of interventions based on Sec.~\ref{sec:invar}, we follow the procedure exemplified in Fig.~\ref{fig:invarInter}. We design two implicit layers with shared parameters $\btheta$, the first layer being unintervened giving the corresponding equilibrium values of the nodes, and the second one being invariantly intervened, for a fixed value of Lie intervention $u$ on the intervened node. In the intervened layer, we replace putative incoming arrows from the invariant node by arrows from the same node in the unintervened graph (as this replacement encodes the invariance assumption) and we replace the functional assignment of the auxiliary node by a Multi-Layer Perceptron (MLP), relying on universal approximation properties to learn a soft intervention that satisfies invariance. We use a least square loss between the intervened and unintervened equilibrium values of the invariant node in order to train the MLP.

\section{Experiments}\label{sec:exp}
The following toy and semi-synthetic experiments illustrate how our framework contributes addressing sustainability challenges exposed in Sec.~\ref{sec:mrio}. 
\paragraph{Evaluation of equilibrium estimation.} We first evaluate the performance of equilibrium layers in computing an accurate estimate of the SSCM solution $\bx^*$. For that we use the SSCM associated to the economic equilibrium of equation~(\ref{eq:leontief}) where we select a subset of sectors in order to vary the dimension of $A$. The full matrices $A$, as well as the final demands $\by$ are estimated from the Exiobase 3 dataset \citep{stadler2018exiobase} for years 2012-2018, using the \textit{Pymrio} library \citep{stadler2021pymrio} for five countries (France, Germany, Italy, USA, Great-Britain). We compare Anderson acceleration (see \citep{walker2011anderson}) for two different choices of the mixing parameter $\beta$, together with the baseline forward iteration approach that simply consists in iterating $\bx_{k+1}=\Bf(\bx_k)$. For each choice of dimension and fixed-point iteration algorithm, we compute the relative error
$
\frac{\|\bx^*-\Bf_{\btheta}(\bx^*)\|}{\|\bx^*\|}\,.
$
The results, averaged across countries and years, show that although forward iteration is the most accurate in lower dimensions, Anderson acceleration with a relaxation parameter $\beta=2.0$ performs better for SSCM dimensions larger than 50. Interestingly, Anderson acceleration with $\beta=1.0$ gives the worst performance, suggesting an appropriate choice of $\beta$ is key. 

\paragraph{Optimization of multiplicative Lie interventions.} In order to investigate \textit{Challenge 1}, we optimize the IO demand driven model of eq.~(\ref{eq:leontief}). The matrices $A$ and $R$, as well as the final demands $\by$ and sector output at equilibrium $\bx^*$ are estimated from yearly activity available in the Exiobase 3 dataset \citep{stadler2018exiobase}, using the \textit{Pymrio} library \citep{stadler2021pymrio}. While the data describes economic interactions across multiple countries, we design an economic equilibrium model of each country by neglecting those interactions, and extracting the blocs of matrices $A$ and $R$ relevant to the country under consideration. We design a distributed multiplicative Lie intervention on the activity of all 200 sectors of the database. The coefficient vector $\boldsymbol{\alpha}$ is then optimized in order to reduce the overall greenhouse gas (GHG) emissions cumulated across sectors (estimated by one component of the stressor vector $\bs$), while enforcing that the overall employment distribution over the sectors stays closest to the non-intervened economy, in order to mitigate challenges associated to reorganizing of economic activities (e.g. mass unemployment and the need for large scale professional reorientation programs). Using the $\ell_1$ norm for regularization, this leads to the following loss:
\begin{equation}\label{eq:joboptim}
\mathcal{L}(\boldsymbol{u}) = \boldsymbol{c}^\top \bx^{(u)} + \lambda \|\be^{(u)}-\be^*\|_1
\end{equation}
where  $\boldsymbol{c}$ is the GHG emission intensity of each sector, and $\be^{(u)}$ and $\be^*$ the intervened and unintervened distributions of employment across sectors (estimated by entry wise multiplication of $\bx^*$ with one row of matrix $R$). 
The graphs shown in Figs.~\ref{fig:countOpt}-\ref{fig:countOptDE} (top), illustrate the trade off between employment preservation and GHG emission reduction achieved by varying $\lambda$ for two different countries. Interestingly, the left tail of these curves reflect differences across countries, with Germany having less room than France for reducing emissions before starting reducing employment significantly. The sectors yielding the largest employment reduction also differ across countries, likely influenced both by the overall structure of each economy. 
%
%
%

%

\paragraph{Control of rebound effects.}
\begin{figure}
	\begin{subfigure}{.49\linewidth}
		\includegraphics[width=\linewidth]{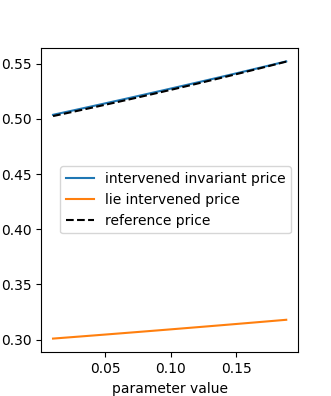}
		\subcaption{\label{fig:invarprice}}
	\end{subfigure}
	\hfill
	\begin{subfigure}{.49\linewidth}
		\includegraphics[width=\linewidth]{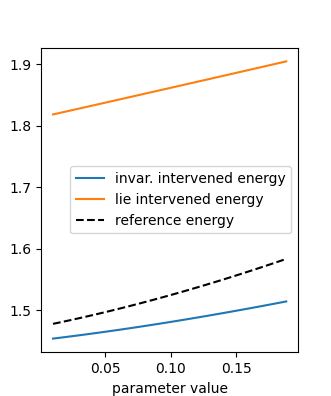}
		\subcaption{\label{fig:invarenergy}}
	\end{subfigure}
	\caption{Outcome of (non-invariant) Lie and invariant interventions on energy efficiency, compared to reference (unintervened) values, in the rebound model described in Fig.~\ref{fig:pricemod}: (a) unit price of the target sector, (b) total energy demand.\label{fig:priceExp}}
\end{figure}

\begin{figure*}
	\hspace*{1cm}
	\begin{subfigure}{.4\linewidth}
		\includegraphics[width=\linewidth]{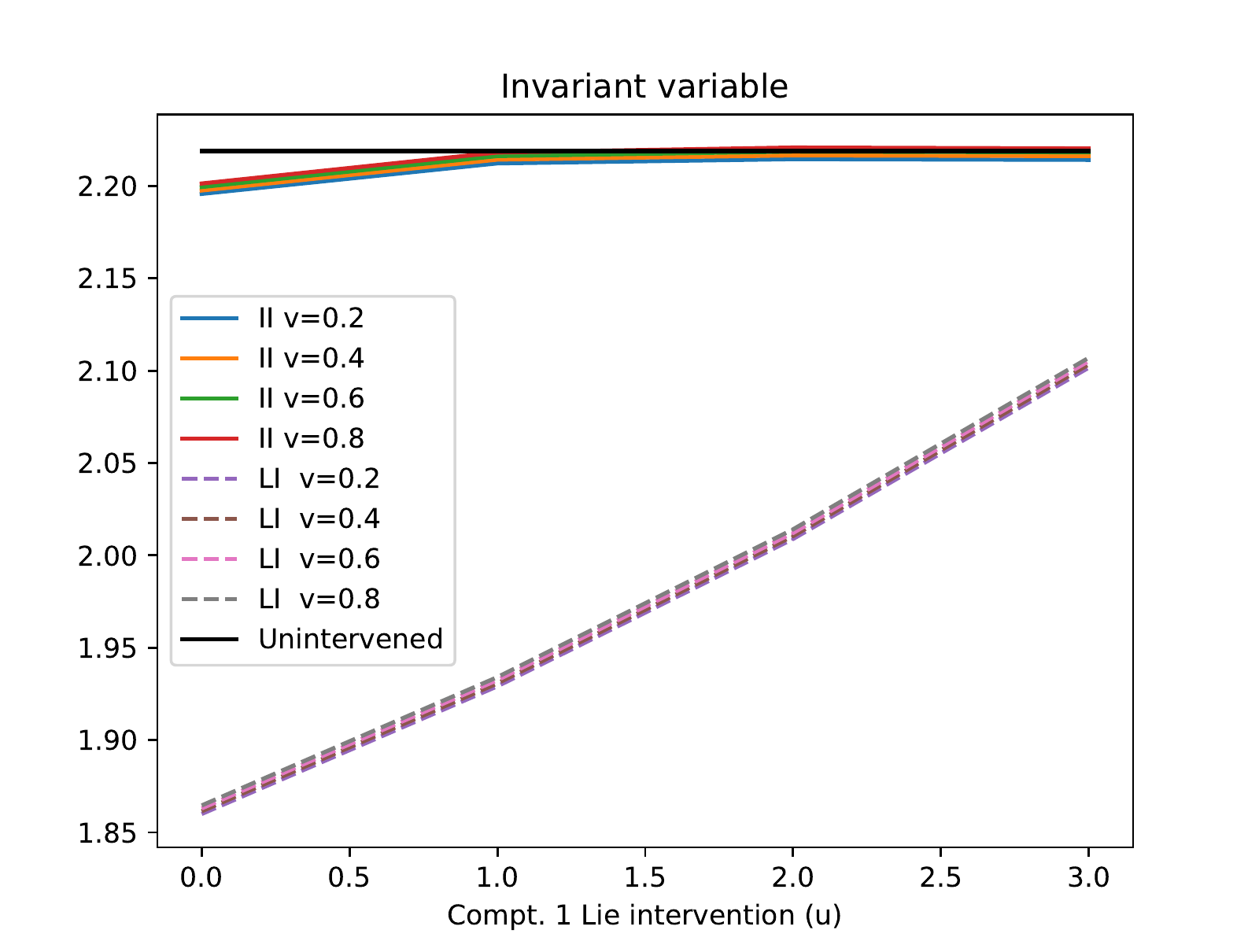}
		\subcaption{\label{fig:compart_invar}}
	\end{subfigure}
	\hfill
	\begin{subfigure}{.4\linewidth}
		\includegraphics[width=\linewidth]{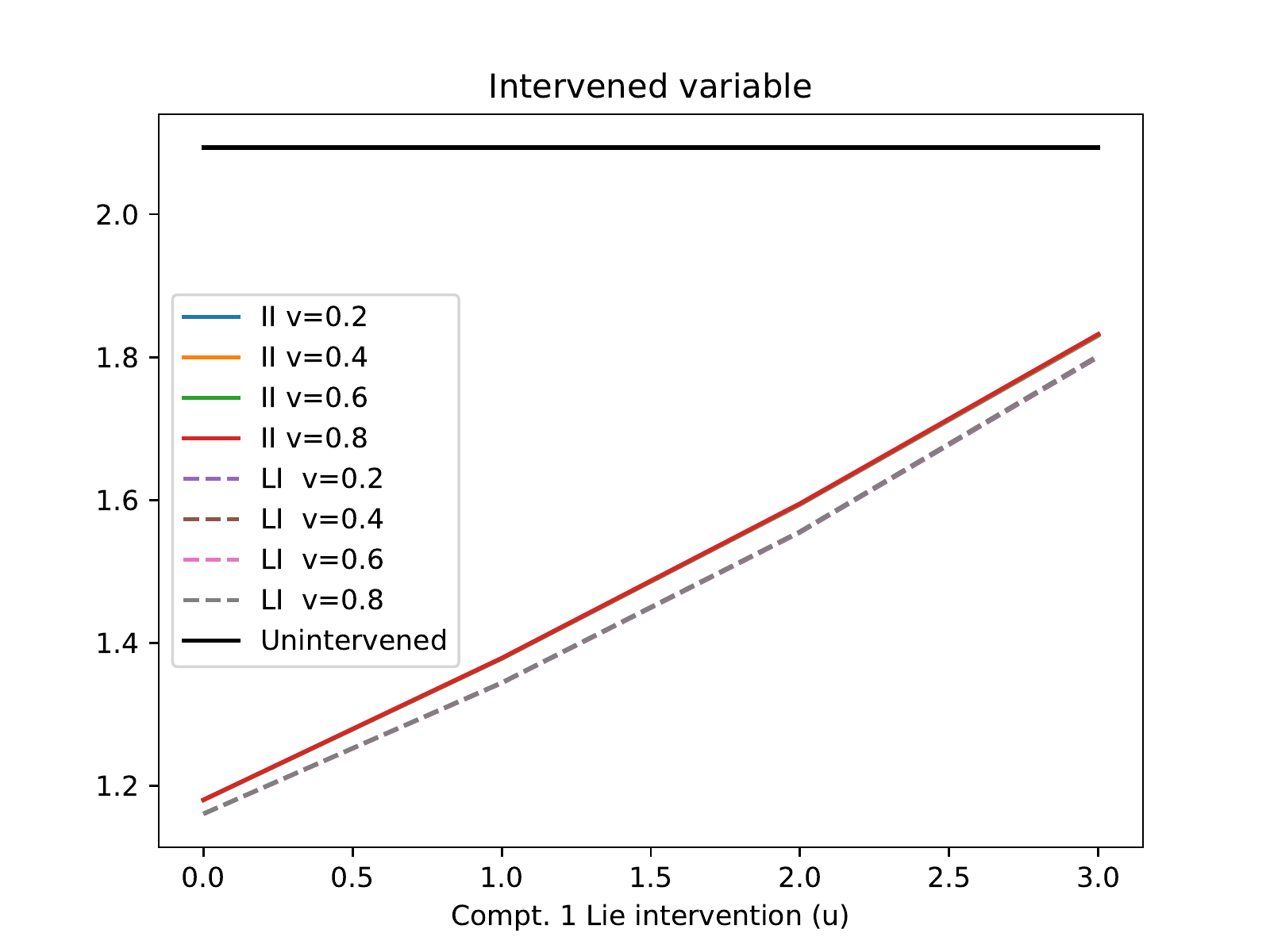}
		\subcaption{\label{fig:compart_inter}}
	\end{subfigure}
\hspace*{1cm}
	\caption{ (a-b) Outcome of the design of compartmentalized interventions for model of Fig.~\ref{fig:comparInter}. The unintervened node values (in black) are compared to invariant interventions (II, solid lines) and their corresponding Lie intervention (LI, dashed lines) (without enforcing invariance), for multiple values of the Lie interventions' parameters $(u,v)$. (a) Value of the invariant node in compartment 1. (b) Value of the intervened node in compartment 1.
		\label{fig:comparExp}}
\end{figure*}

To illustrate how \textit{Challenge 3} of Sec.~\ref{sec:mrio} can be addressed, we used our invariant intervention framework to prevent price rebound effects. We use a toy 3-sector model, with one energy sector and one target sector for which energy efficiency is increased, modeled by a multiplicative Lie intervention on the energy requirements coefficient of the Leontief matrix. The final demand of this target sector is taken as invariant node, and controlled by softly intervening on it through a modification of the unit price of this sector. The invariant intervention is learnt using an MLP with two hidden layers (see Appendix~\ref{app:meth} for details). Fig.~\ref{fig:pricemod} describes the two quantities that are intervened on: we make a multiplicative Lie intervention on the parameter represented by the red node (energy efficiency), and make sure there is no rebound by making the node
invariant to the drop of energy costs using an adaptive taxing policy. Fig.~\ref{fig:invarprice}-\ref{fig:invarenergy} compares 3 models: unintervened (called ``reference'' in the figure), Lie intervened (without enforcing invariance), and invariantly intervened. 
For a range of one parameter left free in the Leontief matrix, the results show the invariant intervention maintains the price close to the unintervened model (Fig.~\ref{fig:invarprice}), while this price is much lower for the Lie intervention (due to the rebound effect). The benefit of  invariance is demonstrated by the effect on the activity of overall energy demand of the economy (Fig.~\ref{fig:invarenergy}): for the Lie intervention, the rebound through prices leads to the so-called \textit{backfire} scenario: the actual energy savings are negative because usage increased beyond potential savings. In contrast, invariant intervention leads to a reduction of energy demand (relative to the unintervened system), as the rebound through prices is prevented. 

\paragraph{Compartmentalized interventions design.}
We further implement compartmentalized interventions and show its benefits for addressing \textit{Challenge~2} in Sec.~\ref{sec:mrio} in multi-sector economic models. We design a two compartment Leontiev model according to Fig.~\ref{fig:comparInter}. We optimize two invariant interventions, $u$ on compartment 1 and $v$ on compartment 2, to follow the conditions of Prop.~\ref{prop:compart}. The results provided in Fig.~\ref{fig:comparExp}, show  the invariant node of compartment 1 is unchanged by both values of $u$ and $v$  (Fig.~\ref{fig:compart_invar}), while the intervened node of this compartment changes value only as a function of its corresponding intervention $u$ (Fig.~\ref{fig:compart_inter}), in a way similar to the (non-invariant) Lie intervention.

\vspace*{-.3cm}
\section{Discussion}
We discuss here some limitations of our approach.
\paragraph{Linearity of the economic model}
The linear input-output model that we use should be understood as one way of modeling interactions between economic sectors, commonly used in environmental economics. It was chosen for its interpretability, illustrative purpose, practical relevance, and because there are established approaches to estimate parameters from economic data. However, it should not be understood that economic models are always linear. Note also that we combine this model with a non-linear demand mechanism to study rebound effects (see Challenge 3 and Section 5). Overall, moving towards non-linear models, as allowed by our setting, is in line with the development of computational models in economy, and notably Integrated Assessment Models (IAM) investigating the complex interactions between climate change and societies. 
\paragraph{The case of multiple equilibria.}
The equilibrium picked by the equilibrium layer depends on the initialization of the estimate of the equilibrium point in the fixed point iteration algorithm implemented by this layer (this can be described using the notion of ``basin of attraction''). While the theory and algorithms developed in this paper focus on the behavior of the causal model in a neighborhood of an given unintervened equilibrium, a prealable grid search for all equilibria may be performed in the most general setting. This may be avoided for the following reasons. From a theoretical perspective, conditions of existence and uniqueness of equilibria are available for many classical models. For example in our application, the Hawkins–Simon condition guarantees the existence of a non-negative output vector that solves the equilibrium relation \citep{hawkins1949note}. 
From a practical perspective, we are often interested mainly in intervening on the empirically observed equilibrium. For models based on unintervened observed data, we can thus check that the simulated unintervened equilibrium matches the observed data. If however there is a mismatch between the equilibrium obtained by the deep equilibrium layer and the one we are interested in, we can enforce the initialization of the fixed point iteration algorithm in the neighborhood of the expected equilibrium. Our experiments were run with a fixed initialization of the equilibrium point (zero). 

\section{Conclusion}
\vspace*{-.2cm}
We introduced a differentiable soft intervention design framework for general equilibrium systems. We argue those are more likely to approximate deployable interventions in real-world complex systems, e.g. to address key challenges of the transition to sustainable economies. Theoretical results and algorithmic tools are provided to design interventions with desirable invariance properties under the assumption that the considered system is in equilibrium and model parameters are known. Further work in this direction will need to address identifiability of the considered models from observational or experimental data. 


\begin{acknowledgements} 
MB is grateful to Philipp Geiger for insightful discussions. This work was supported by the German Federal Ministry of Education and Research (BMBF): T\"ubingen AI Center, FKZ: 01IS18039B; and by the Machine Learning Cluster of Excellence, EXC number 2064/1 - Project number 390727645. 
\end{acknowledgements}

\bibliography{cyclic}

\begin{thebibliography}{40}
\providecommand{\natexlab}[1]{#1}
\providecommand{\url}[1]{\texttt{#1}}
\expandafter\ifx\csname urlstyle\endcsname\relax
  \providecommand{\doi}[1]{doi: #1}\else
  \providecommand{\doi}{doi: \begingroup \urlstyle{rm}\Url}\fi

\bibitem[Amos et~al.(2018)Amos, Rodriguez, Sacks, Boots, and
  Kolter]{amos2018differentiable}
Brandon Amos, Ivan Dario~Jimenez Rodriguez, Jacob Sacks, Byron Boots, and
  J~Zico Kolter.
\newblock Differentiable {MPC} for end-to-end planning and control.
\newblock \emph{arXiv preprint arXiv:1810.13400}, 2018.

\bibitem[Andersen(2013)]{andersen2013expect}
Holly Andersen.
\newblock When to expect violations of causal faithfulness and why it matters.
\newblock \emph{Philosophy of Science}, 80\penalty0 (5):\penalty0 672--683,
  2013.

\bibitem[Arrobbio and Padovan(2018)]{arrobbio2018}
Osman Arrobbio and Dario Padovan.
\newblock A vicious tenacity: The efficiency strategy confronted with the
  rebound effect.
\newblock \emph{Frontiers in Energy Research}, 6:\penalty0 114, 2018.

\bibitem[Bai et~al.(2019)Bai, Kolter, and Koltun]{bai2019deep}
Shaojie Bai, J~Zico Kolter, and Vladlen Koltun.
\newblock Deep equilibrium models.
\newblock \emph{arXiv preprint arXiv:1909.01377}, 2019.

\bibitem[Besserve et~al.(2018)Besserve, Shajarisales, Sch{\"o}lkopf, and
  Janzing]{besserve2018aistats}
Michel Besserve, Naji Shajarisales, Bernhard Sch{\"o}lkopf, and Dominik
  Janzing.
\newblock Group invariance principles for causal generative models.
\newblock In \emph{AISTATS}, 2018.

\bibitem[Blom and Mooij(2021)]{blom2021causality}
Tineke Blom and Joris~M Mooij.
\newblock Causality and independence in perfectly adapted dynamical systems.
\newblock \emph{arXiv preprint arXiv:2101.11885}, 2021.

\bibitem[Blom et~al.(2020)Blom, van Diepen, and Mooij]{blom2020conditional}
Tineke Blom, Mirthe~M van Diepen, and Joris~M Mooij.
\newblock Conditional independences and causal relations implied by sets of
  equations.
\newblock \emph{arXiv preprint arXiv:2007.07183}, 2020.

\bibitem[Bongers et~al.(2016)Bongers, Forr{\'e}, Peters, Sch{\"o}lkopf, Mooij,
  et~al.]{bongers2016foundations}
Stephan Bongers, Patrick Forr{\'e}, Jonas Peters, Bernhard Sch{\"o}lkopf,
  Joris~M Mooij, et~al.
\newblock Foundations of structural causal models with cycles and latent
  variables.
\newblock \emph{arXiv preprint arXiv:1611.06221}, 2016.

\bibitem[Brockway et~al.(2021)Brockway, Sorrell, Semieniuk, Heun, and
  Court]{brockway2021energy}
Paul~E Brockway, Steve Sorrell, Gregor Semieniuk, Matthew~Kuperus Heun, and
  Victor Court.
\newblock Energy efficiency and economy-wide rebound effects: A review of the
  evidence and its implications.
\newblock \emph{Renewable and Sustainable Energy Reviews}, page 110781, 2021.

\bibitem[Correa and Bareinboim(2020)]{correa2020general}
Juan Correa and Elias Bareinboim.
\newblock General transportability of soft interventions: Completeness results.
\newblock \emph{Advances in Neural Information Processing Systems}, 33, 2020.

\bibitem[Dearing et~al.(2014)Dearing, Wang, Zhang, Dyke, Haberl, Hossain,
  Langdon, Lenton, Raworth, Brown, et~al.]{dearing2014safe}
John~A Dearing, Rong Wang, Ke~Zhang, James~G Dyke, Helmut Haberl, Md~Sarwar
  Hossain, Peter~G Langdon, Timothy~M Lenton, Kate Raworth, Sally Brown, et~al.
\newblock Safe and just operating spaces for regional social-ecological
  systems.
\newblock \emph{Global Environmental Change}, 28:\penalty0 227--238, 2014.

\bibitem[Eberhardt and Scheines(2007)]{eberhardt2007interventions}
Frederick Eberhardt and Richard Scheines.
\newblock Interventions and causal inference.
\newblock \emph{Philosophy of science}, 74\penalty0 (5):\penalty0 981--995,
  2007.

\bibitem[El~Ghaoui et~al.(2021)El~Ghaoui, Gu, Travacca, Askari, and
  Tsai]{el2021implicit}
Laurent El~Ghaoui, Fangda Gu, Bertrand Travacca, Armin Askari, and Alicia Tsai.
\newblock Implicit deep learning.
\newblock \emph{SIAM Journal on Mathematics of Data Science}, 3\penalty0
  (3):\penalty0 930--958, 2021.

\bibitem[Esteban-Bravo(2004)]{esteban2004computing}
Mercedes Esteban-Bravo.
\newblock Computing equilibria in general equilibrium models via interior-point
  methods.
\newblock \emph{Computational Economics}, 23\penalty0 (2):\penalty0 147--171,
  2004.

\bibitem[Geiger and Straehle(2020)]{geiger2020learning}
Philipp Geiger and Christoph-Nikolas Straehle.
\newblock Learning game-theoretic models of multiagent trajectories using
  implicit layers.
\newblock \emph{arXiv preprint arXiv:2008.07303}, 2020.

\bibitem[Haberl et~al.(2019)Haberl, Wiedenhofer, Pauliuk, Krausmann,
  M{\"u}ller, and Fischer-Kowalski]{haberl2019contributions}
Helmut Haberl, Dominik Wiedenhofer, Stefan Pauliuk, Fridolin Krausmann,
  Daniel~B M{\"u}ller, and Marina Fischer-Kowalski.
\newblock Contributions of sociometabolic research to sustainability science.
\newblock \emph{Nature Sustainability}, 2\penalty0 (3):\penalty0 173--184,
  2019.

\bibitem[Hawkins and Simon(1949)]{hawkins1949note}
David Hawkins and Herbert~A Simon.
\newblock Note: some conditions of macroeconomic stability.
\newblock \emph{Econometrica, Journal of the Econometric Society}, pages
  245--248, 1949.

\bibitem[Hickel and Kallis(2020)]{hickel2020green}
Jason Hickel and Giorgos Kallis.
\newblock Is green growth possible?
\newblock \emph{New political economy}, 25\penalty0 (4):\penalty0 469--486,
  2020.

\bibitem[Imbens and Rubin(2015)]{imbens2015causal}
Guido~W Imbens and Donald~B Rubin.
\newblock \emph{Causal inference in statistics, social, and biomedical
  sciences}.
\newblock Cambridge University Press, 2015.

\bibitem[Jaber et~al.(2020)Jaber, Kocaoglu, Shanmugam, and
  Bareinboim]{jaber2020causal}
Amin Jaber, Murat Kocaoglu, Karthikeyan Shanmugam, and Elias Bareinboim.
\newblock Causal discovery from soft interventions with unknown targets:
  Characterization and learning.
\newblock \emph{Advances in neural information processing systems},
  33:\penalty0 9551--9561, 2020.

\bibitem[Jakob and Edenhofer(2014)]{jakob2014green}
Michael Jakob and Ottmar Edenhofer.
\newblock Green growth, degrowth, and the commons.
\newblock \emph{Oxford Review of Economic Policy}, 30\penalty0 (3):\penalty0
  447--468, 2014.

\bibitem[Jevons(1866)]{jevons1866coal}
William~S Jevons.
\newblock \emph{The coal question}.
\newblock Routledge, 1866.

\bibitem[Johari et~al.(2022)Johari, Li, Liskovich, and
  Weintraub]{johari2022experimental}
Ramesh Johari, Hannah Li, Inessa Liskovich, and Gabriel~Y Weintraub.
\newblock Experimental design in two-sided platforms: An analysis of bias.
\newblock \emph{Management Science}, 2022.

\bibitem[Kocaoglu et~al.(2019)Kocaoglu, Jaber, Shanmugam, and
  Bareinboim]{kocaoglu2019characterization}
Murat Kocaoglu, Amin Jaber, Karthikeyan Shanmugam, and Elias Bareinboim.
\newblock Characterization and learning of causal graphs with latent variables
  from soft interventions.
\newblock \emph{Advances in Neural Information Processing Systems}, 32, 2019.

\bibitem[Lee(2013)]{lee2013smooth}
John~M Lee.
\newblock Smooth manifolds.
\newblock In \emph{Introduction to Smooth Manifolds}, pages 1--31. Springer,
  2013.

\bibitem[Leontief(1951)]{leontief1951structure}
Wassily~W Leontief.
\newblock The structure of american economy, 1919-1939: an empirical
  application of equilibrium analysis.
\newblock Technical report, 1951.

\bibitem[Mooij et~al.(2013)Mooij, Janzing, and Sch{\"o}lkopf]{MooJanSch13}
Joris Mooij, Dominik Janzing, and Bernhard Sch{\"o}lkopf.
\newblock From ordinary differential equations to structural causal models: the
  deterministic case.
\newblock In \emph{Proceedings of the Twenty-Ninth Conference Annual Conference
  on Uncertainty in Artificial Intelligence}, pages 440--448, Corvallis, OR,
  2013. AUAI Press.

\bibitem[Oei et~al.(2020)Oei, Hermann, Herpich, Holtem{\"o}ller,
  L{\"u}nenb{\"u}rger, and Schult]{oei2020coal}
Pao-Yu Oei, Hauke Hermann, Philipp Herpich, Oliver Holtem{\"o}ller, Benjamin
  L{\"u}nenb{\"u}rger, and Christoph Schult.
\newblock Coal phase-out in germany--implications and policies for affected
  regions.
\newblock \emph{Energy}, 196:\penalty0 117004, 2020.

\bibitem[Paszke et~al.(2019)Paszke, Gross, Massa, Lerer, Bradbury, Chanan,
  Killeen, Lin, Gimelshein, Antiga, Desmaison, Kopf, Yang, DeVito, Raison,
  Tejani, Chilamkurthy, Steiner, Fang, Bai, and Chintala]{NEURIPS2019_9015}
Adam Paszke, Sam Gross, Francisco Massa, Adam Lerer, James Bradbury, Gregory
  Chanan, Trevor Killeen, Zeming Lin, Natalia Gimelshein, Luca Antiga, Alban
  Desmaison, Andreas Kopf, Edward Yang, Zachary DeVito, Martin Raison, Alykhan
  Tejani, Sasank Chilamkurthy, Benoit Steiner, Lu~Fang, Junjie Bai, and Soumith
  Chintala.
\newblock Pytorch: An imperative style, high-performance deep learning library.
\newblock In \emph{Advances in Neural Information Processing Systems 32}, pages
  8024--8035. Curran Associates, Inc., 2019.

\bibitem[Pearl(2000)]{pearl2000causality}
Judea Pearl.
\newblock \emph{Causality: models, reasoning and inference}, volume~29.
\newblock Cambridge Univ Press, 2000.

\bibitem[Peters et~al.(2017)Peters, Janzing, and Sch\"olkopf]{causality_book}
Jonas Peters, Dominik Janzing, and Bernhard Sch\"olkopf.
\newblock \emph{Elements of Causal Inference -- Foundations and Learning
  Algorithms}.
\newblock MIT Press, 2017.

\bibitem[Peters et~al.(2020)Peters, Bauer, and Pfister]{peters2020causal}
Jonas Peters, Stefan Bauer, and Niklas Pfister.
\newblock Causal models for dynamical systems.
\newblock \emph{arXiv preprint arXiv:2001.06208}, 2020.

\bibitem[Rothenh{\"a}usler et~al.(2015)Rothenh{\"a}usler, Heinze, Peters, and
  Meinshausen]{rothenhausler2015backshift}
Dominik Rothenh{\"a}usler, Christina Heinze, Jonas Peters, and Nicolai
  Meinshausen.
\newblock Backshift: Learning causal cyclic graphs from unknown shift
  interventions.
\newblock \emph{Advances in Neural Information Processing Systems}, 28, 2015.

\bibitem[Sherwood and Huber(2010)]{sherwood2010adaptability}
Steven~C Sherwood and Matthew Huber.
\newblock An adaptability limit to climate change due to heat stress.
\newblock \emph{Proceedings of the National Academy of Sciences}, 107\penalty0
  (21):\penalty0 9552--9555, 2010.

\bibitem[Stadler(2021)]{stadler2021pymrio}
Konstantin Stadler.
\newblock Pymrio--a python based multi-regional input-output analysis toolbox.
\newblock \emph{Journal of Open Research Software}, 9\penalty0 (1), 2021.

\bibitem[Stadler et~al.(2018)Stadler, Wood, Bulavskaya, S{\"o}dersten, Simas,
  Schmidt, Usubiaga, Acosta-Fern{\'a}ndez, Kuenen, Bruckner,
  et~al.]{stadler2018exiobase}
Konstantin Stadler, Richard Wood, Tatyana Bulavskaya, Carl-Johan S{\"o}dersten,
  Moana Simas, Sarah Schmidt, Arkaitz Usubiaga, Jos{\'e} Acosta-Fern{\'a}ndez,
  Jeroen Kuenen, Martin Bruckner, et~al.
\newblock Exiobase 3: Developing a time series of detailed environmentally
  extended multi-regional input-output tables.
\newblock \emph{Journal of Industrial Ecology}, 22\penalty0 (3):\penalty0
  502--515, 2018.

\bibitem[Walker and Ni(2011)]{walker2011anderson}
Homer~F Walker and Peng Ni.
\newblock Anderson acceleration for fixed-point iterations.
\newblock \emph{SIAM Journal on Numerical Analysis}, 49\penalty0 (4):\penalty0
  1715--1735, 2011.

\bibitem[Wallenborn(2018)]{wallenborn2018}
Gr{\'egoire} Wallenborn.
\newblock Rebounds are structural effects of infrastructures and markets.
\newblock \emph{Frontiers in Energy Research}, 6:\penalty0 99, 2018.

\bibitem[Wiebe et~al.(2018)Wiebe, Bjelle, T{\"o}bben, and
  Wood]{wiebe2018implementing}
Kirsten~Svenja Wiebe, Eivind~Lekve Bjelle, Johannes T{\"o}bben, and Richard
  Wood.
\newblock Implementing exogenous scenarios in a global {MRIO} model for the
  estimation of future environmental footprints.
\newblock \emph{Journal of Economic Structures}, 7\penalty0 (1):\penalty0
  1--18, 2018.

\bibitem[Wood et~al.(2018)Wood, Moran, Stadler, Ivanova, Steen-Olsen,
  Tisserant, and Hertwich]{wood2018}
Richard Wood, Daniel Moran, Konstantin Stadler, Diana Ivanova, Kjartan
  Steen-Olsen, Alexandre Tisserant, and Edgar~G. Hertwich.
\newblock Prioritizing consumption-based carbon policy based on the evaluation
  of mitigation potential using input-output methods.
\newblock \emph{Journal of Industrial Ecology}, 22\penalty0 (3):\penalty0
  540--552, 2018.

\end{thebibliography}

\appendix
\onecolumn

\section{Additional background}\label{app:back}
\subsection{Smooth manifolds}
While many non-equivalent definitions exist for smooth manifold, we follow \cite{lee2013smooth} in defining smoothness as infinite continuously differentiability of functions. A diffeomorphism is then a smooth bijection whose inverse is also smooth.

For an n-dimensional topological manifold $M$, an atlas is a collection of coordinate charts $(U_k,\varphi_k)$ such that $U_k$'s are open sets of $M$ covering it, and such that the mappings $\varphi_k: U_k\mapsto \varphi_k(U_k)\subset \mathbb{R}^n$ are homeomorphisms (continuous bijection with continuous inverse). Briefly, the atlas is smooth whenever $\varphi_k \circ \varphi_n^{-1}$ is are diffeomorphisms whenever well defined, and  a smooth manifold is a topological manifold associated with a maximal smooth atlas.

A smooth map $F: M\rightarrow N$ between two smooth manifolds $M$ and $N$ is a function such that for any chart $(U,\varphi)$ and $(V,\psi)$, $\psi\circ F \circ \varphi^{-1}$ is smooth whenever well defined.

\subsection{Lie groups}
We first provide a formal definition of groups. 
\begin{definition}[Group]\label{def:grp}
	A set $G$ is a group if it is equipped with a binary operation $``\cdot'':G\times G\rightarrow G$ satisfying
	\begin{enumerate}
		\item Associativity: $\forall a,b,c\in G$, $(a\cdot b)\cdot c = a\cdot(b\cdot c)$
		\item Identity: There exists $e\in G$ such that $\forall a \in G$, $a\cdot e = e \cdot a = a$.
		\item Inverse: $\forall a\in G$, there exists $b\in G$ such that $a\cdot b = b\cdot a = e$. This inverse is denoted $a^{-1}$.
	\end{enumerate}
\end{definition}
Then a Lie group is essentially a group that is also a smooth manifold. 
\begin{definition}[Lie Group]
	A Lie Group $G$ is a nonempty set satisfying the following conditions:
	\begin{itemize}
		\item $G$ is a group.
		\item $G$ is a smooth manifold.
		\item The group operation $\cdot:G\times G \rightarrow G$ and the inverse map $.^{-1}:G\rightarrow G$ are smooth.
	\end{itemize}
\end{definition}

We are often interested in sets of transformations, which respect a group structure, but are applied to objects that are not necessarily group elements. 
This can be studied through group actions, which describe how groups \emph{act} on other mathematical entities. 
\begin{definition}[Lie group Action]\label{def:grp_act}
	Given a Lie group $G$ and a set $X$, a Lie group action (or smooth group action) is a function $\cdot_X:G\times X \rightarrow X$ such that the following conditions are satisfied.
	\begin{enumerate}
		\item Identity: If $e\in G$ is the identity element, then $e\cdot_X x = x$, $\forall x \in X$.
		\item Compatibility: $\forall g,h \in G$ and $\forall x \in X$, $g\cdot_X (h\cdot_X x) = ((g\cdot h)\cdot_X x)$
		\item Smoothness: the map $\cdot_X:G\times X \rightarrow X$ is smooth.  
	\end{enumerate}
\end{definition}

\subsection{Cyclic causal models}
A classical type of hard interventions are \textit{perfect interventions}, which replace the structural assignments of a given variable $X_k$ by an assignment $X_k\coloneqq \xi_k$, with $\xi_k$ constant \citep{blom2020conditional}. It thus eliminates the arrows in the causal graph pointing to this variables, and makes this variable deterministic. 

In particular, tracing the effects of perfect interventions requires special assumptions. In contrast, soft interventions may be read from the so-called causal ordering graph, which can be built from the original SCM graph. Broadly construed, a unique causal ordering graph can be constructed with several algorithms \citep{blom2020conditional}. This is a directed cluster graph that contains groups of variables connected by oriented edges (starting from single variable in a given cluster, and pointing to another cluster). By construction, the resulting graph between clusters entailed by these edges is directed and contains no cycles. As a consequence, the effect of generic soft intervention on clustered variables can be easily read from this graph. 


\subsection{Link between equilibrium and dynamic models}
The equilibrium of eq.~(\ref{eq:leontief}) can be thought of as the asymptotic value of $\bx$ in a dynamic model (see~Appendix~\ref{app:back})
\[
\frac{d\bx}{dt} = A \bx +\by - \bx\,,
\]
where the increase or decrease of the sectors' activity is controlled by the imbalance between their demand $A\bx+y$ and their current output $\bx$. More generally, any fixed point-equation can be thought of a the equilibrium value of some dynamical system, for example by considering a numerical algorithm that converges to it. However, the relationship between dynamical systems and self-consistent equation is not one to one. Notably, we can rescale the time evolution of a stable dynamical system to create many other that converge to the same self-consistent equation. Moreover, by inverting the arrow of time, we can obtain systems for which the self-consistent equation is an unstable equilibrium. As mentioned in main text, in this work we leave aside the dynamical aspects to focus on the equilibrium properties.

\subsection{MRIO models}
Multi-regional input-output models are built based on macro-economic information, notably the one provided by the National Accounts of the countries involved in the model. 
The technical coefficient matrix of eq.~(\ref{eq:leontief}) is computed from so-called \textit{Supply and Use Tables} that form the basis of National Accounts.
The unit used to measure output is frequently monetary (e.g., EUR) due to the data collection process and to allow an homogeneous treatment of the economic flows. However, under homogeneity and linearity assumptions, the output of each sector may be converted in appropriate physical units using unit prices and material flow data. Moreover, there also exist hybrid MRIO models which include information regarding physical flows in the economy (energy, raw materials, ...) and the are combined with monetary information to ensure the best level of self-consistency.

\section{Proof of main text results}\label{app:proofs}
\subsection{Proof of Proposition~\ref{prop:localSolv}}
\begin{proof}
	Assuming the SSCM is locally diffeomorphic entails that the Jacobian of $\bx\rightarrow \bx-\Bf(\bx,\btheta^{\rf})$ is invertible at $\bx=\bx^{\rf}$. Then the Jacobian of   $(\bx,\btheta)\rightarrow (\bx-\Bf(\bx,\btheta),\,\btheta)$ is also invertible at $(\bx^{\rf},\btheta^{\rf})$ (due to its block triangular structure). Using the inverse function theorem for smooth maps between smooth manifolds \cite[Theorem 4.5]{lee2013smooth}, this implies that there exists connected open neighborhoods $(U,V)$ of $(\bx^{\rf},\btheta^{\rf})$ and $(\textbf{0},\btheta^{\rf})$ such that  
	\begin{align*}
		g \colon   \phantom{++} U & \rightarrow  V \\
		(\bx,\btheta) & \mapsto  (\bx-\Bf(\bx,\btheta),\,\btheta)
	\end{align*}
	is a diffeomorphism. As a consequence, self-consistent solutions $(\bx,\btheta)$ in $U$ are given by $S=g^{-1}((\{0\}\times \mathcal{T})\cap V)$. It is a submanifold of same dimension as $\mathcal{T}$ for the following reasons:
	\begin{itemize}
		\item $S$ is a manifold diffeomorphic to $(\{0\}\times \mathcal{T})\cap V$ and thus has the same dimension \citep[Theorem 2.17]{lee2013smooth},
		\item $(\{0\}\times \mathcal{T})\cap V$ is an open submanifold because $V$ is open, and thus has the same dimension as $\{0\}\times \mathcal{T}$ \citep[Proposition 5.1]{lee2013smooth}
		\item $\{0\}\times \mathcal{T}$ has the same dimension as $\Tcal$ because it is diffeomorphic to it \citep[Propositions 5.3 and 2.17]{lee2013smooth}.
	\end{itemize} 
	Let us now define the cartesian projection
	\begin{align*}
		\pi \colon \phantom{++} U & \rightarrow \mathcal{T}\\
		(\bx,\btheta) & \mapsto  \btheta\,,
	\end{align*}
	we want to establish that there exist an open neighborhood $U_{\btheta}$ of $\btheta^{\rf}$ such that there is a unique self-consistent solution for each parameter choice in this set
	$\pi_{|S}$ is a smooth embedding because it is an injective smooth immersion, and is open\footnote{$\pi_{|S}$ is open because $\pi_{|S}\circ g^{-1}_{|(\{0\}\times \mathcal{T})\cap V}$ is a smooth submersion and thus open by Proposition 4.28 in \cite{lee2013smooth}, and $g_S$ is also open as the restriction of a diffeomorphism. 
	}, by \citet[Proposition 4.22 ]{lee2013smooth}). As a consequence $\pi (S)$ is an embedded submanifold of $\Tcal$ diffeomorphic to $S$ (by \citet[Proposition 5.2]{lee2013smooth}). Since we have shown that the dimension of $S$ is the dimension of $\Tcal$, then $\pi (S)$ is a submanifold of same codimension $0$ (same dimension as its ambient manifold) and is thus an open submanifold of $\Tcal$ (Proposition 5.1 in \cite{lee2013smooth}). As a consequence, $\pi (S)$ is open, such that there is an open neighborhood of $\btheta^{\rf}$ included in it. 
	Then for any parameter chosen in this neighborhood, there is one solution to the self-consistency equation, by definition of the image. Assume there are two distinct solution for this parameter, then the mapping  $(\bx,\btheta)\rightarrow (\bx-\Bf(\bx,\btheta),\,\btheta)$ would not be a diffeomorphism.
\end{proof}

\subsection{Proof of Proposition~\ref{prop:liesolv}}
\begin{proof}[Proof]
	We extend the smooth parameterization of function $f$ by $\btheta$ to get a smooth parameterization of the intervened functional assignments by $\bar{\btheta}=(g,\btheta)$. Indeed, the mapping
	\[
	(x,\bar{\btheta})\mapsto g\cdot f(x,\theta)
	\]
	is smooth as a composition of 
	the following smooth maps
	\[
	(\bx,\btheta,g) \underset{f \mbox{ smooth }}{\mapsto} (f(\bx,\btheta),g) \underset{\varphi \mbox{ smooth }}\mapsto \varphi(g,f(\bx,\btheta)) = g\cdot f(\bx,\btheta)
	\]
	where the smoothness of each transformation stem from the definition of SSCM and Lie interventions, respectively.
	Proposition~\ref{prop:localSolv} applied around the extend parameter $(e,\btheta^{\rf})$ implies that there exists a   neighborhood $U_{(e,\btheta^{\rf})}$ of this point such that the intervened solution is uniquely solvable and the mapping from the extended parameter to the solution is smooth. There exists moreover a product neighborhood $ U_L\times U_{\btheta} \subset U_{(e,\btheta^{\rf})}$ (this is a basic property of neighborhoods on product spaces). 
	By continuity of the partial derivative of the intervened functional assignment (due to smoothness of the Lie group action), dependency on the parents of the intervened variables is preserved in a neighborhood of the identity, such that the intervention is soft in the considered neighborhood. 
\end{proof}

\subsection{Proof of Propostion~\ref{prop:invar}}
\begin{proof}
	The Lie intervention parameterized by $u$ guaranties solvability of the SSCM is preserved in a neighborhood of the identity (Proposition~\ref{prop:liesolv}), and we denote $x^{(u)}(\btheta)$ the unique solution in such neighborhood, with $x^{(e)}(\btheta)=x^{*}(\btheta)$. 
	The Jacobian $J^{\btheta}_{x^*_{\parents_k}}(\btheta^{\rf})$ is the Jacobian of the mapping from the parameters $\btheta$ to the vector consisting of the parent nodes of $k$ at equilibrium. Because this Jacobian is full column rank, there exists  a neighborhood of $e$ such that for any fixed $u$ in it, the mapping $\btheta \mapsto \bx_{\parents_k}^{(u)}(\btheta)$ is injective in a neighborhood of the reference parameter. As a consequence the restriction to its image is a diffeomorphic map between manifolds. Let us denote  $\psi^{(u)}$ its inverse. 
	
	Consider the SSCM obtained by performing a hard intervention $x_j\coloneqq x_j^*(\btheta)$. Because the original SSCM is locally diffeomorphic at $(\bx^{\rf},\btheta^{\rf})$, $\{x_j\coloneqq x_j^*(\btheta)\}$ is a smooth assignment, and because additionally the Jacobian of the mapping $\bx_{-j}\rightarrow \bx_{-j}-\Bf_{-j}(\bx_{-j},\btheta^{\rf})$ is invertible, then this hard intervened system is also locally diffeomorphic at  $(\bx^{\rf},\btheta^{\rf})$ (exploiting the block diagonality of the Jacobian of its assignment). As a consequence, Lie intervention with parameter $u$ on node $i$ of this (already hard-intervened) system leads to a smooth intervened equilibrium $x^{(u)}$.
	
	Let us recall that the partial derivative $\frac{\partial x^*_j}{\partial x_k }_{|\btheta=\btheta^{\rf}}$ corresponds to the derivative with respect to the hard interventions value. 
	The assumption $\frac{\partial x^*_j}{\partial x_k }_{|\btheta=\btheta_{\rf}}\neq 0$ thus entails, by the inverse function theorem, that there exists also a smooth mapping $\phi^{(u)}$ such that $x^{(u)}_k = \phi^{(u)}(x^*_j(\btheta))$ in a neighborhood of $(e, \btheta_{\rf})$. As a consequence, the mapping defined as $f^{(u)}_k = \phi^{(u)}\circ x_j^*(.)\circ \psi^{(u)}$ is a soft intervention replacing $f_k$ achieving the same equilibrium values as the above hard-intervened system under Lie interventions, and in particular the invariance constraint $x^{(u)}_j(\btheta)=x^{*}_j(\btheta)$.
\end{proof}

\subsection{Proof of Proposition~\ref{prop:compart}}
\begin{proof}
	We proceed iteratively by adding one intervention after the next. First intervention on compartment $C_1$ leaves invariant the equilibrium values of the remaining compartments $C_{-1}$ as the only node from $C_{1}$ influencing them is invariant. 
	
	Given $C_1,...,C_n$ satisfy invariance with respect to each others interventions, consider intervening on $C_{n+1}$. As $C_{n+1}$ receives only inputs from intervened upon compartments $C_1,...,C_n$ through invariant nodes, the invariant intervention on it can be designed identical as for the non-intervened system. Moreover, invariance of the nodes having outgoing arrows to other compartments ensures that the equilibrium values of other (potentially intervened upon) compartments $C_{-(n+1)}$ remains invariant. 
	\michel{same as above proof, be explicit about the interventions (which equations are replaced, which variables are taken a parameters relative to )} 
\end{proof}
\michel{add condition for existence of invariant intervention on each node, given distributed parameters and restricted connectivity, those are weaker than the general definition. }

\section{Additional theoretical results}\label{app:add}
\subsection{Motivating example of Sec.~\ref{sec:invar}}
Let us restate the unintervened assignments of this example.

\begin{eqnarray*}
	\setlength{\jot}{0pt}
	x &= &\tau \,,\\
	y &= & (\alpha x +\beta z) \,, \\
	z &=&  \gamma y \,.
\end{eqnarray*}
The equilibrium solution then writes

\begin{eqnarray*}
	\setlength{\jot}{0pt}
	x^* &= &\tau \,,\\
	y^* &= & \frac{\alpha \tau}{1-\beta \gamma} \,, \\
	z^* &=&  \frac{\gamma \alpha \tau}{1-\beta \gamma} \,.
\end{eqnarray*}

Applying multiplicative Lie interventions on both $x$ an $y$ leads to the assignments
\begin{eqnarray*}
	\setlength{\jot}{0pt}
	x &= &\tau \,,\\
	y &= &u_y (\alpha x +\beta z) \,, \\
	z &=& u_z \gamma y \,.
\end{eqnarray*}
which leads to the intervened equilibrium
\begin{eqnarray*}
	\setlength{\jot}{0pt}
	x^{(\boldsymbol{u})} &= &\tau \,,\\
	y^{(\boldsymbol{u})} &= & \frac{u_y\alpha \tau}{1-u_z u_y\beta \gamma} \,, \\
	z^{(\boldsymbol{u})} &=&  \frac{u_y u_z\gamma \alpha \tau}{1-u_y u_z\beta \gamma} \,.
\end{eqnarray*}
We can thus notice that choosing $u_z = \frac{1}{u_y}$ makes the intervened equilibrium value invariant for any choice of parameters $(\tau,\alpha,\beta,\gamma)$.

\section{Methods}\label{app:meth}
Following \citep{bai2019deep}, we implemented implicit layers using the pyTorch library \citep{NEURIPS2019_9015}. We use Anderson acceleration with $m=5$ previous iterations and a relaxation parameter  $\beta=2.0$ (based on preliminary analysis) to compute iteratively the fixed points of the implicit layers, for both forward and backward passes, with a maximum number of iterations of 5000 and a tolerance of .0001. Our experiments were run with a fixed initialization of the equilibrium point (zero). 

Optimization of interventions is done using backpropagation with adaptive moment estimation (Adam), with a learning rate of .001 and 10000-20000 iterations. 
Soft interventions to enforce invariance are learned with two hidden layer perceptrons, with 20 and 10 hidden units respectively for the first and second layers, and all layers have ReLU activation functions. At each iteration, free parameters $\btheta$ are sampled from an factorized Gaussian distribution whose means and variances are chosen to cover the neighborhood of the reference point. Optimization of invariant soft interventions is performed by sampling at each iteration from a range of values of the Lie intervention and unintervened model parameters.

Toy experiments use artificially generated parameters for few sectors economic models (this is the case of paragraphs ``Control of rebound effects'' and ``Compartmentalized intervention design''), whose structure is described in Figs.~2a and 1b. Instead, the semi-synthetic experiments use 200 sector economic models based on realistic parameters from the Exiobase3 dataset, as described in the paragraphs ``Evaluation of equilibrium estimation'', and ``Optimization of multiplicative Lie interventions''.

Code for the toy optimization experiments is provided at  \href{https://github.com/mbesserve/lie-inter}{https://github.com/mbesserve/lie-inter}. 

\section{Supplemental discussion}\label{sec:supdisc}
\paragraph{Socio-economic impacts of environmental policies}
In full generality, whether and which environmental policies have a negative socio-economic impact is a highly debated topic. In main text, we argue that straightforward measures that could be taken to significantly contribute to achieve environmental goals typically have a short-term socio-economic cost. The example repetitively used in our paper is activity reduction of greenhouse gas emitting sectors, wish has straightforward \textit{short-term impacts} on their employment \citep{oei2020coal}. We are not aware of literature challenging the view that such classical socio-economic and environmental goals are at least to some extent at odds and require tradeoff from the standpoint of political decision makers. \textit{On the longer term}, the feasibility of making these goals compatible based on economic concepts such as Green Growth is debated \citep{jakob2014green,hickel2020green}.



\end{document}